\documentclass{article} 
\PassOptionsToPackage{semicolon}{natbib}
\usepackage{iclr2023_conference,times}
\iclrfinalcopy

\usepackage[utf8]{inputenc} 
\usepackage[T1]{fontenc}    

\usepackage{titletoc}
\usepackage[toc, page, header]{appendix} 

\usepackage[colorlinks=true, linkcolor=blue, citecolor=blue,urlcolor=black]{hyperref}
\usepackage{url}            
\usepackage{booktabs}       
\usepackage{amsfonts}       
\usepackage{nicefrac}       
\usepackage{microtype}      
\usepackage{xcolor}         
\usepackage{footnote}
\usepackage{multirow}
\usepackage{pifont}
\newcommand{\cmark}{\ding{51}}%
\newcommand{\xmark}{\ding{55}}%

\title{Variance-Aware Sparse Linear Bandits
}

\author{%
  Yan Dai \\
  IIIS, Tsinghua University\\
  \texttt{yan-dai20@mails.tsinghua.edu.cn} \\
  \AND
  Ruosong Wang \\
  University of Washington\\
  \texttt{ruosongw@cs.washington.edu}\\
   \And
   Simon S. Du \\
   University of Washington \\
   \texttt{ssdu@cs.washington.edu} \\
}

\usepackage{amsmath}
\usepackage{amsthm}
\usepackage{amssymb}
\usepackage{mathtools}

\usepackage{algorithm}
\usepackage[noend]{algpseudocode}
\usepackage{cleveref}
\usepackage{comment}
\usepackage{bbm}
\usepackage{enumitem}

\algrenewcommand\algorithmicrequire{\textbf{Input:}}
\algrenewcommand\algorithmicensure{\textbf{Output:}}
\makeatletter
\algnewcommand{\LineComment}[1]{\Statex \hskip\ALG@thistlm #1}
\makeatother

\newtheorem{theorem}{Theorem}
\newtheorem{lemma}[theorem]{Lemma}
\newtheorem{proposition}[theorem]{Proposition}

\newtheorem{corollary}[theorem]{Corollary}
\newtheorem{definition}[theorem]{Definition}

\newcommand{\daiyan}[1]{{\color{violet} [Yan: #1]}}

\renewcommand{\daiyan}[1]{}

\newcommand{\E}{\operatornamewithlimits{\mathbb{E}}}
\newcommand{\trans}{{\mathsf{T}}}

\renewcommand{\O}{\operatorname{\mathcal{O}}}
\newcommand{\Otil}{\operatorname{\tilde{\mathcal O}}}

\Crefname{assumption}{Assumption}{Assumptions}
\Crefformat{equation}{Eq. #2(#1)#3}
\Crefrangeformat{equation}{Eqs. #3(#1)#4 to #5(#2)#6}
\Crefmultiformat{equation}{Eqs. #2(#1)#3}{ and #2(#1)#3}{, #2(#1)#3}{ and #2(#1)#3}
\Crefrangemultiformat{equation}{Eqs. #3(#1)#4 to #5(#2)#6}{ and #3(#1)#4 to #5(#2)#6}{, #3(#1)#4 to #5(#2)#6}{ and #3(#1)#4 to #5(#2)#6}

\renewcommand{\bar}{\overline}
\renewcommand{\hat}{\widehat}
\renewcommand{\tilde}{\widetilde}
\newcommand{\polylog}{\operatorname{\mathrm{polylog}}}

\allowdisplaybreaks

\begin{document}
\normalsize

\maketitle

\begin{abstract}
It is well-known that for sparse linear bandits, when ignoring the dependency on sparsity which is much smaller than the ambient dimension, the worst-case minimax regret is $\widetilde{\Theta}\left(\sqrt{dT}\right)$ where $d$ is the ambient dimension and $T$ is the number of rounds.
On the other hand, in the benign setting where there is no noise and the action set is the unit sphere, one can use divide-and-conquer to achieve $\widetilde{\mathcal O}(1)$ regret, which is (nearly) independent of $d$ and $T$.
In this paper, we present the first variance-aware regret guarantee for sparse linear bandits:  $\widetilde{\mathcal O}\left(\sqrt{d\sum_{t=1}^T \sigma_t^2} + 1\right)$, where $\sigma_t^2$ is the variance of the noise at the $t$-th round.
This bound naturally interpolates the regret bounds for the worst-case constant-variance regime (i.e., $\sigma_t \equiv \Omega(1)$) and the benign deterministic regimes (i.e., $\sigma_t \equiv 0$).
To achieve this variance-aware regret guarantee, we develop a general framework that converts any variance-aware linear bandit algorithm to a variance-aware algorithm for sparse linear bandits in a ``black-box'' manner. Specifically, we take two recent algorithms as black boxes to illustrate that the claimed bounds indeed hold, where the first algorithm can handle unknown-variance cases and the second one is more efficient.

  
\end{abstract}

\section{Introduction}\label{sec:introduction}

This paper studies the sparse linear stochastic bandit problem, which is a special case of linear stochastic bandits. 
In linear bandits~\citep{dani2008stochastic}, the agent is facing a sequential decision-making problem lasting for $T$ rounds.
For the $t$-th round,  the agent chooses an action $ x_t\in \mathcal X \subseteq \mathbb{R}^d$, where $\mathcal X$ is an action set, and receives a noisy reward $r_t=\langle {\theta^\ast}, x_t\rangle+\eta_t$ where ${\theta^\ast}\in \mathcal X$ is the (hidden) parameter of the game and $\eta_t$ is random zero-mean noise.
The goal of the agent is to minimize her \textit{regret} $\mathcal R_T$, that is, the difference between her cumulative reward $\sum_{t=1}^T \langle \theta^\ast, x_t\rangle$ and $\max_{ x \in \mathcal{X}}\sum_{t=1}^T\langle \theta^\ast, x\rangle$ (check \Cref{eq:regret} for a definition).
\citet{dani2008stochastic} proved that the minimax optimal regret for linear bandits is $\tilde{\Theta}(d\sqrt T)$ when the noises are independent Gaussian random variables with means $0$ and variances $1$ and both $\theta^*$ and the actions $x_t$ lie in the unit sphere in $\mathbb{R}^d$.%
\footnote{Throughout the paper, we will use the notations $\tilde{\mathcal O}(\cdot)$ and $\tilde{\Theta}(\cdot)$  to hide $\log T,\log d,\log s$ (where $s$ is the sparsity parameter, which will be introduced later) and $\log \log \frac 1\delta$ factors (where $\delta$ is the failure probability).\label{footnote:Otilde}}

In real-world applications such as recommendation systems, only a few features may be relevant despite a large candidate feature space. In other words, the high-dimensional linear regime may actually allow a low-dimensional structure. As a result, if we still use the linear bandit model, we will always suffer $\Omega(d\sqrt T)$ regret no matter how many features are useful. Motivated by this, the sparse linear stochastic bandit problem was introduced \citep{abbasi2012online,carpentier2012bandit}. 
This problem has an additional constraint that the hidden parameter, ${\theta^\ast}$, is sparse, i.e., $\lVert {\theta^\ast}\rVert_0\le s$ for some $s \ll d$.
However, the agent has \textit{no} prior knowledge about $s$ and thus the interaction protocol is exactly the same as that of linear bandits. The minimax optimal regret for sparse linear bandits is $\tilde{\Theta}(\sqrt{sdT})$ \citep{abbasi2012online,antos09sparse}.\footnote{
\citet{carpentier2012bandit} and \citet{lattimore2015linear} obtained an $\O(s\sqrt T)$ regret bound under different models. The former one assumed a component-wise noise model, while the latter one assumed a $\lVert \theta^\ast\rVert_1\le 1$ ground-truth as well as a $\lVert x_t\rVert_\infty\le 1$ action space.
See \Cref{sec:appendix related work} for more discussions on this.}
This bound bypasses the $\Omega(d\sqrt{T})$ lower bound for linear bandits as we always have $s=\lVert {\theta^\ast}\rVert_0\le d$ and the agent does not have access to $s$ either (though a few previous works assumed a known $s$).

However, both the $\Otil(d\sqrt{T})$ and the $\Otil(\sqrt{sdT})$ bounds are the \textit{worst-case} regret bounds and sometime are too pessimistic especially when $d$ is large. 
On the other hand, many problems with delicate structures permit a regret bound much smaller than the worst-case bound. 
The structure this paper focuses on is the magnitude of the noise.
Consider the following motivating example.

\textbf{Motivating Example (Deterministic Sparse Linear Bandits).} Consider the case where the action set is the unit sphere $\mathcal{X} = \mathbb S^{d-1}$, and  there is no noise, i.e., the feedback is $r_t=\langle {\theta^\ast}, x_t\rangle$ for each round $t\in [T]$.
In this case, one can identify all non-zero entries of $ \theta^*$ coordinates in ${\mathcal O}(s\log d)$ steps with high probability via a divide-and-conquer algorithm, and thus yield a \emph{dimension-free} regret $\Otil(s)$ (see \Cref{sec:appendix divide and conquer} for more details about this).%
\footnote{\label{footnote:action set}We also remark that some assumptions on the action is needed. For example, if every action can only query one coordinate (each action corresponds to one vector of the standard basis) then an $\Omega\left(d\right)$ regret lower bound is unavoidable. Hence, in this paper, we only consider the benign case that action set is the unit sphere.}
However, this divide-and-conquer algorithm is specific for deterministic sparse linear bandit problems and does not work for noisy models.
Henceforth, we study the following natural question:\begin{center}
\textbf{\textit{Can we design an algorithm whose regret adapts to the noise level such that the regret interpolates the  $\sqrt{dT}$-type bound in the worst case and the dimension-free bound in the deterministic case?
}}
\end{center}

Before introducing our results, we would like to mention that there are recent works that studied the noise-adaptivity in linear bandits~\citep{zhou2021nearly,zhang2021improved,kim2021improved}.
They gave \emph{variance-aware} regret bounds of the form $\widetilde{\mathcal O}\left (\text{poly}(d)\sqrt{\sum_{t=1}^T \sigma_t^2} + \text{poly}(d)\right )$ where $\sigma_t^2$ is the (conditional) variance of the noise $\eta_t$. This bound reduces to the standard $\widetilde{\mathcal{O}}(\text{poly}(d)\sqrt{T})$ bound in the worst-case when $\sigma_t=\Omega(1)$,  and to a constant-type regret $\Otil(\text{poly}(d))$ that is independent of $T$.
However, compared with the linear bandits setting, the variance-aware bound for sparse linear bandits is more significant because it reduces to a \textit{dimension-free} bound in the noiseless setting. Despite this, to our knowledge, no variance-aware regret bounds exist for sparse linear bandits.

\subsection{Our Contributions}
This paper gives the first set of variance-aware regret bounds for sparse linear bandits.
We design a general framework, \texttt{VASLB}, to reduce variance-aware sparse linear bandits to variance-aware linear bandits with little overhead in regret. 
For ease of presentation, we define the following notation to characterize the \textit{variance-awareness} of a sparse linear bandit algorithm:
\begin{definition}
A variance-aware sparse linear bandit algorithm $\mathcal F$ is \textit{$(f(s,d),g(s,d))$-variance-aware}, if for any given failure probability $\delta > 0$, with probability $1-\delta$, $\mathcal F$ ensures
\begin{equation*}
    \mathcal R_T^{\mathcal F}\le \tilde{\mathcal O}\left (f(s,d)\sqrt{\sum_{t=1}^T \sigma_t^2}\polylog \frac 1\delta+g(s,d)\polylog \frac 1\delta\right ),
\end{equation*}
where $\mathcal R_T^{\mathcal F}$ is the regret of $\mathcal F$ in $T$ rounds, $d$ is the ambient dimension and $s$ is the maximum number of non-zero coordinates. Specifically, for linear bandits, $f,g$ are functions only of $d$.
\end{definition}


Hence, an $(f,g)$-variance-aware algorithm will achieve $\tilde{\mathcal O}(f(s,d)\sqrt T\polylog \frac 1\delta)$ worst-case regret and $\tilde{\mathcal O}(g(s,d)\polylog \frac 1\delta)$ deterministic-case regret. Ideally, we would like $g(s,d)$ being independent of $d$, making the bound \textit{dimension-free} in deterministic cases, as the divide-and-conquer approach.

In this paper, we provide a general framework that can convert \textit{any} linear bandit algorithm $\mathcal F$ to a corresponding sparse linear bandit algorithm $\mathcal G$ in a black-box manner. 
Moreover, it is \textit{variance-aware-preserving}, in the sense that, if $\mathcal F$ enjoys the variance-aware property, so does $\mathcal G$.
Generally speaking, if the plug-in linear bandit algorithm $\mathcal F$ is $(f(d),g(d))$-variance-aware, then our framework directly gives an $(s(f(s)+\sqrt d),s(g(s)+1))$-variance-aware algorithm $\mathcal G$ for sparse linear bandits.

Besides presenting our framework, we also illustrate its usefulness by plugging in two existing variance-aware linear bandit algorithms, where the first one is variance-aware (i.e., works in unknown-variance cases) but computationally inefficient. In contrast, the second one is efficient but requires the variance $\sigma_t^2$ to be delivered together with feedback $r_t$.
Their regret guarantees are stated as follows.
\setlist{nolistsep}
\begin{enumerate}[leftmargin=*,noitemsep]
    \item The first variance-aware linear bandit algorithm we plug in is \texttt{VOFUL}, which was proposed by \citet{zhang2021improved} and improved by \citet{kim2021improved}. This algorithm is computationally inefficient but deals with \textit{unknown variances}.
    Using this \texttt{VOFUL}, our framework generates a $(s^{2.5}+s\sqrt d,s^3)$-variance-aware algorithm for sparse linear bandits.
    Compared to the $\Omega(\sqrt{sdT})$ regret lower-bound for sparse linear bandits \citep[\S 24.3]{lattimore2020bandit}, our worst-case regret bound is near-optimal up to a factor $\sqrt s$. Moreover, our bound is independent of $d$ and $T$ in the deterministic case, nearly matching the bound of divide-and-conquer algorithm dedicated to the deterministic setting up to $\text{poly}(s)$ factors.
    \item The second algorithm we plug in is \texttt{Weighted OFUL} \citep{zhou2021nearly}, which requires known variances but is computationally efficient. We obtain an $(s^2+s\sqrt d,s^{1.5}\sqrt{T})$-variance-aware \textit{efficient} algorithm. In the deterministic case, this algorithm can only achieve a $\sqrt{T}$-type regret bound (albeit still independent of $d$). We note that this is not due to our framework but due to \texttt{Weighted OFUL} which itself cannot gives constant regret bound in the deterministic setting.

\end{enumerate}

Moreover, we would like to remark that our deterministic regret can be further improved if a better variance-aware linear bandit algorithm is deployed: The current ones either have $\Otil(d^2)$ \citep{kim2021improved} or $\Otil(\sqrt{dT})$ \citep{zhou2021nearly} regret in the deterministic case, which are both sub-optimal compared with the $\Omega(d)$ lower bound.








\subsection{Related Work}

\textbf{Linear Bandits.} This problem was first introduced by \citet{dani2008stochastic}, where an algorithm with regret $\O(d\sqrt T(\log T)^{\nicefrac 32})$ and a near-matching regret lower-bound $\Omega(d\sqrt T)$ were given. After that, an improved upper bound $\O(d\sqrt{T}\log T)$ \citep{abbasi2011improved} together with an improved lower bound $\Omega(d\sqrt{T\log T})$ \citep{li2019nearly} were derived. An extension of it, namely linear contextual bandits, where the action set allowed for each step can vary with time \citep{chu2011contextual,kannan2018smoothed,li2019nearly,li2021tight}, is receiving more and more attention. The best-arm identification problem where the goal of the agent is to approximate $\theta^\ast$ with as few samples as possible \citep{soare2014best,degenne2019non,jedra2020optimal,alieva2021robust} is also of great interest.

\textbf{Sparse Linear Bandits.} \citet{abbasi2011improved} and \citet{carpentier2012bandit} concurrently considered the sparse linear bandit problem, where the former work assumed a noise model of $r_t=\langle  x_t,{\theta^\ast}\rangle+\eta_t$ such that $\eta_t$ is $R$-sub-Gaussian and achieved $\Otil(R\sqrt{sdT})$ regret, while the latter one considered the noise model of $r_t=\langle  x_t+ \eta_t,{\theta^\ast}\rangle$ such that $\lVert \eta_t\rVert_2\le \sigma$ and $\lVert \theta^\ast\rVert_2\le \theta$, achieving $\Otil((\sigma+\theta)^2s\sqrt T)$ regret. 
\citet{lattimore2015linear} assumed an  hypercube (i.e., $\mathcal X=[-1,1]^d$) action set and a $\lVert {\theta^\ast}\rVert_1\le 1$ ground-truth, yielding $\Otil(s\sqrt T)$ regret. 
\citet{antos09sparse} proved a $\Omega(\sqrt{dT})$ lower-bound when $s=1$ with the unit sphere as $\mathcal X$.
Some recent works considered data-poor regimes where $d\gg T$ \citep{hao2020high,hao2021information,hao2021online,wang2020nearly}, which is beyond the scope of this paper.
Another work worth mentioning is the recent work by \citet{dong2021provable}, which studies bandits or MDPs with deterministic rewards. Their result implies an $\Otil(T^{\nicefrac{15}{16}}s^{\nicefrac{1}{16}})$ bound for deterministic sparse linear bandits, which is independent of $d$. They also provided an ad-hoc divide-and-conquer algorithm, which achieves $\O(s\log d)$ regret only for deterministic cases.

\textbf{Variance-Aware Online Learning.}
For tabular MDPs, the variance information is widely used in both discounted settings \citep{lattimore2012pac} and episodic settings \citep{azar2017minimax,jin2018q}, where \citet{zanette2019tighter} used variance information to derive problem-dependent regret bounds for tabular MDPs.
For bandits, \citet{audibert2009exploration} made use of variance information in multi-armed bandits, giving an algorithm outperforming existing ones when the variances for suboptimal arms are relatively small.
For bandits with high-dimensional structures, \citet{faury2020improved} studied variance adaptation for logistic bandits,
\citet{zhou2021nearly} considered linear bandits and linear mixture MDPs where the variance information is revealed to the agent, giving an $\Otil(d\sqrt{\sum_{t=1}^T \sigma_t^2}+\sqrt{dT})$ guarantee for linear bandits,
and \citet{zhang2021improved} proposed another algorithm for linear bandits and linear mixture MDPs, which does not require any variance information, whose regret can be improved to be $\Otil(d^{1.5}\sqrt{\sum_{t=1}^T \sigma_t^2}+d^2)$ as shown by \citet{kim2021improved}.
The recent work by \citet{hou2022almost} considered variance-constrained best arm identification, where the feedback noise only depends on the action by the agent (whereas ours can depend on time, which is more general than theirs).
Another recent work \citep{zhao2022bandit} studied variance-aware regret bounds for bandits with general function approximation in the known variance case.

\textbf{Stochastic Contextual Sparse Linear Bandits.} In the setting, the action set for each round $t$ is i.i.d. sampled (called the ``context''). It is known that $\Otil(\sqrt{sT})$ regret is achievable in this setting \citep{kim2019doubly,ren2020dynamic,oh2021sparsity,ariu2022thresholded}. However, in our setting where both action set $\mathcal X$ and ground-truth $\theta^\ast$ are fixed, a polynomial dependency on $d$ is in general unavoidable because it is impossible to learn more than one parameter per arm \citep{bastani2020online}, agreeing with the $\Omega(\sqrt{dT})$ lower bound when $s=1$ \citep{antos09sparse,abbasi2012online}.

\begin{table}[t]
\begin{minipage}{\textwidth}
\vspace{-0.5cm}
    \caption{An overview of the proposed algorithms/results and comparisons with related works.}
    \label{tab:related work}
    \begin{savenotes}
    \renewcommand{\arraystretch}{1.5}
    \resizebox{\textwidth}{!}{%
    \begin{tabular}{|c|c|c|c|c|c|}\hline
    Algorithm & Setting & Worst-case Regret \footnote{``Worst-case'' means the variances $\sigma_t^2$ are all $1$. Here, $d$ is the ambient dimension, $T$ is the number of rounds, and $s$ is the sparsity parameter (only applicable to sparse linear bandits).} & Deterministic-case Regret \footnote{``Deterministic-case'' means the variances $\sigma_t^2$ are all $0$. Only applicable to variance-aware algorithms.} & Efficiency & Variances \\\hline
    $\texttt{ConfidenceBall}_2$ \citep{dani2008stochastic} & LinBandit & $\Otil(d\sqrt T)$  & \multirow{4}{*}{N/A} & \cmark & \multirow{4}{*}{N/A}\\\cline{1-3} \cline{5-5}
    \texttt{OFUL} \citep{abbasi2012online} & Sparse LinBandit & $\Otil(\sqrt{sdT})$ & & \cmark & \\\cline{1-3} \cline{5-5}
    \texttt{SL-UCB} \citep{carpentier2012bandit} & Sparse LinBandit & $\Otil(s\sqrt{T})$ \footnote{With a different feedback model; see \Cref{sec:appendix related work} for more comparison.} & & \cmark & \\\cline{1-3} \cline{5-5}
    \citet[Algorithm 4]{lattimore2015linear} & Sparse LinBandit & $\Otil(s\sqrt{T})$ \footnote{With a different action set and an different assumption on $\theta^\ast$; see \Cref{sec:appendix related work} for more comparison.} & & \cmark & \\\hline
    \texttt{Weighted OFUL} \citep{zhou2021nearly} & LinBandit & $\Otil(d\sqrt{T})$ & $\Otil(\sqrt{dT})$ & \cmark & Known \\\hline
    \texttt{VOFUL2} \citep{kim2021improved} & LinBandit & $\Otil(d^{1.5}\sqrt{T})$ & $\Otil(d^2)$ & \xmark & Unknown \\\hline
    \multirow{2}{*}{\texttt{VASLB} (\textbf{This work})}
    & Sparse LinBandit & $\Otil(s^2\sqrt{T}+s\sqrt{dT})$ & $\Otil(s^{1.5}\sqrt T)$ & \cmark & Known \\\cline{2-6}
    & Sparse LinBandit & $\Otil(s^{2.5}\sqrt{T}+s\sqrt{dT})$ & $\Otil(s^3)$  & \xmark & Unknown\\\hline
    Lower Bound \citep{antos09sparse} & Sparse LinBandit & $\Omega(\sqrt{dT})$ \footnote{This bound holds even if $s=1$ and the action set is fixed to be the unit sphere.} & N/A & N/A & N/A\\\hline
    \end{tabular}}
    \end{savenotes}
\end{minipage}
\end{table}

\section{Problem Setup}

\textbf{Notations.} We use $[N]$ to denote the set $\{1,2,\ldots,N\}$ where $N\in \mathbb N$. For a vector $ x\in \mathbb R^d$, we use $\lVert x\rVert_p$ to its $L_p$-norm, namely $\lVert  x\rVert_p\triangleq (\sum_{i=1}^d x_i^p)^{\nicefrac 1p}$.
We use $\mathbb S^{d-1}$ to denote the $(d-1)$-dimensional unit sphere, i.e., $\mathbb S^{d-1}\triangleq \{ x\in \mathbb R^d\mid \lVert  x\rVert_2=1\}$.
We use $\tilde{\mathcal O}(\cdot)$ and $\tilde{\Theta}(\cdot)$ to hide all logarithmic factors in $T,s,d$ and $\log \frac 1\delta$ (see \Cref{footnote:Otilde}).
For a random event $\mathcal E$, we denote its indicator by $\mathbbm 1[\mathcal E]$.

We assume the action space and the ground-truth space are both the $(d-1)$-dimensional unit sphere, denoted by $\mathcal X\triangleq \mathbb S^{d-1}$. Denote the ground-truth by $ \theta^\ast\in \mathcal X$. There will be $T\ge 1$ rounds for the agent to make decisions sequentially. At the beginning of round $t\in [T]$, the agent has to choose an action $ x_t\in \mathcal X$. At the end of step $t$, the agent receives a \textit{noisy} feedback $r_t=\langle  x_t, \theta^\ast\rangle+\eta_t$, $\forall t\in [T]$, where $\eta_t$ is an independent zero-mean Gaussian random variable. Denote by $\sigma_t^2=\text{Var}(\eta_t)$ the variance of $\eta_t$. For a fair comparison with non-variance-aware algorithms, we assume that $\sigma_t^2\le 1$.
The agent then receives a (deterministic and unrevealed) reward of magnitude $\langle x_t, \theta^\ast\rangle$ for this round.

The agent is allowed to make the decision $x_t$ based on all historical actions $x_1, \ldots, x_{t-1}$, all historical feedback $r_1,\ldots,r_{t-1}$, and any amount of private randomness. The agent's goal is to minimize the regret, defined as follows.
\begin{definition}[Regret]
The following random variable is the \textbf{regret} of a linear bandit algorithm:
\begin{equation}\label{eq:regret}
\mathcal R_T=\max_{ x\in \mathcal X}\sum_{t=1}^T  \langle  x, \theta^\ast\rangle-\sum_{t=1}^T \langle  x_t, \theta^\ast\rangle=\sum_{t=1}^T \langle  \theta^\ast- x_t, \theta^\ast\rangle,
\end{equation}
where the second equality is due to our assumption that $\mathcal X=\mathbb S^{d-1}$.
\end{definition}

For the sparse linear bandit problem, we have an additional restriction that $\lVert {\theta^\ast}\rVert_0\le s$, i.e., there are at most $s$ coordinates of ${\theta^\ast}$ is non-zero. However, as mentioned in the introduction, the agent does \textit{not} know anything about $s$ -- she only knows that she is facing a (probably sparse) linear environment.

\section{Framework and Analysis}

\begin{algorithm}[!t]
\caption{Variance-Aware Sparse Linear Bandits (\texttt{VASLB}) Framework}
\label{alg:framework}
\begin{algorithmic}[1]
\Require{Number of dimensions $d$, linear bandit algorithm $\mathcal F$ and its regret estimator $\bar {\mathcal R_n^{\mathcal F}}$}
\State {Initialize gap threshold $\Delta\gets \frac 14$, estimated ``good'' coordinates $S\gets \varnothing$, current round $t\gets 0$.}
\While{$t<T$} \Comment{The algorithm automatically increases $t$ by $1$ per query.}
\If{$S\ne \varnothing$}\Comment{Have some coordinates to ``commit''.}
\State{Initialize a new linear bandit instance $\mathcal F$ on coordinates $S$.}\Comment{{\color{blue} ``Commit'' phase.}}
\State{Execute $\mathcal F$ for ${n_\Delta^a}\ge 1$ steps \& maintain pessimistic estimation $\bar{\mathcal R_{n_\Delta^a}^{\mathcal F}}$, until $\frac 1{n_\Delta^a}\bar{\mathcal R_{n_\Delta^a}^{\mathcal F}}<\Delta^2$. \label{line:phase a terminiate}}
\State{Suppose that $\mathcal F$ plays $ x_1, x_2,\ldots, x_{n_\Delta^a}$. Set $\hat \theta=\frac 1{n_\Delta^a}\sum_{i=1}^{n_\Delta^a}  x_i$ as the estimate for $\{\theta_i^\ast\}_{i\in S}$.\label{line:estimate theta by commit}}
\EndIf
\If{$\sum_{i\in S}\hat \theta_i^2\le 1-\Delta^2$\label{line:skipping}}\Comment{Still have undiscovered coordinates with $\theta_i^\ast>\frac \Delta 2$}
\State{Let $R\gets \sqrt{1-\sum_{i\in S}\hat \theta_i^2}$, $K=d-\lvert S\rvert$.}  \Comment{{\color{blue} ``Explore'' phase.}}
\State Perform ${n_\Delta^b}\ge 1$ calls to \Call{RandomProjection}{$K,R,S,\hat \theta$} in \Cref{alg:random projection}, until
\begin{align}
    &2\sqrt{2\sum_{k=1}^{n_\Delta^b} (r_{k,i}-\bar r_i)^2\ln \frac 4\delta}<{n_\Delta^b}\cdot \frac{\Delta}{4},\quad \forall 1\le i\le K,\label{eq:designing N in Option II}
\end{align}

\LineComment{where $r_k$ is the $k$-th return vector of \Call{RandomProjection}{} and $\bar r\triangleq \frac 1{n_\Delta^b}\sum_{k=1}^{n_\Delta^b} r_k$.}
\For{$i=1,2,\ldots,K$}
\State \textbf{if} $\lvert \bar r_i\rvert>\Delta$ where $\bar r=\frac 1{n_\Delta^b}\sum_{k=1}^{n_\Delta^b}  r_k$ \textbf{then} add the $i$-th element that is not in $S$ to $S$.
\EndFor
\EndIf
\EndWhile
\end{algorithmic}
\end{algorithm}

\begin{algorithm}[!t]
\caption{The \textsc{RandomProjection} Subroutine}\label{alg:random projection}
\begin{algorithmic}[1]
\Function{RandomProjection}{$K,R,S,\hat \theta$}
\State Generate $K$ i.i.d. samples $y_1,y_2,\ldots,y_K$, each with equal probability being $\pm \frac{R}{\sqrt K}$.
\State Play $x\in \mathcal X$ constructed as $x_{i}=\begin{cases}\hat \theta_i,&\quad i\in S\\y_j,&\quad i\text{ is the $j$-th element that is not in $S$}\end{cases}$
\State \Return{$\frac{K}{R^2}((r-\sum_{i\in S}\hat \theta_i^2)~ y)$ where $r=\langle x, \theta^\ast\rangle+\eta$ is the (noisy) feedback.}
\EndFunction
\end{algorithmic}
\end{algorithm}

Our framework \texttt{VASLB} is presented in \Cref{alg:framework}. We explain its design in \Cref{sec:technical overview} and sketch its analysis in \Cref{sec:analysis}. Then we give two applications using \texttt{VOFUL2} \citep{kim2021improved} and \texttt{Weighted VOFUL} \citep{zhou2021nearly} as $\mathcal F$, whose analyses are sketched in \Cref{sec:zhang et al,sec:zhou et al}.




\subsection{Main Difficulties and Technical Overview}\label{sec:technical overview}


At a high level, our framework follows the spirit of the classic ``explore-then-commit'' approach (which is directly adopted by \citet{carpentier2012bandit}), where the agent first identifies those ``huge'' entries of $\theta^\ast$ and then performs a linear bandit algorithm on them. 
However, it is hard to incorporate variances into this vanilla idea to make it variance-aware -- the desired regret depends on variances and is thus unknown to the agent. Thus it is difficult to determine a ``gap threshold'' $\Delta$ (that is, the agent stops to ``commit'' after identifying all $\theta_i^\ast\ge \Delta$) within a few rounds.
For example, in the deterministic case, the agent must identify all non-zero entries to make the regret independent of $T$; on the other hand, in the worst case where $\sigma_t\equiv 1$, the agent only needs to identify all entries with magnitude at least $T^{-1/4}$ to yield $\sqrt{T}$-style regret bounds.
At the same time, the actual setting might be mixture of them (e.g., $\sigma_t\equiv 0$ for $t\le t_0$ and $\sigma_t\equiv 1$ for $t>t_0$ where $t_0\in [T]$). As a result, such an idea cannot always succeed in determining the correct threshold $\Delta$ and getting the desired regret.

In our proposed framework, we tackle this issue by ``explore-then-commit'' multiple times. We reduce the uncertainty gently and alternate between ``explore'' and ``commit'' modes.
We decrease a ``gap threshold'' $\Delta$ in a halving manner and, at the same time, maintain a set $S$ of coordinates that we believe to have a magnitude larger than $\Delta$. For each $\Delta$, we ``explore'' (estimating $\theta_i^\ast$ and adding those greater than $\Delta$ into $S$) and ``commit'' (performing linear bandit algorithms on coordinates in $S$).

However, as we ``explore'' again after ``committing'', we face a unique challenge: Suppose {that some entry $i\in [d]$ is identified to be at least $2\Delta$ by previous ``explore'' phases}. During the next ``explore'' phase, we cannot directly do pure exploration {over the remaining unidentified coordinates} -- otherwise, coordinate $i$ will incur $4\Delta^2$ regret for each round. Fortunately, we can get an estimation $\hat \theta_i$ of $\theta_i^\ast$ during the previous ``commit'' phase thanks to the regret-to-sample-complexity conversion (\Cref{eq:regret-to-sample-complexity}). Guarded with this estimation, we can reserve $\hat \theta_i$ mass for arm $i$ and subtract $\hat \theta_i^2$ from the feedback in subsequent ``explore'' phases. More preciously, we do the following.
\setlist{nolistsep}
\begin{enumerate}[leftmargin=*,noitemsep]
    \item In the ``commit'' phase where we apply the black-box $\mathcal F$, we estimate $\{\theta_i^\ast\}_{i\in S}$ by the regret-to-sample-complexity conversion:
    Suppose $\mathcal F$ plays $ x_1, x_2,\ldots, x_n$ and achieves regret $\mathcal R_n^{\mathcal F}$, then
    \begin{equation}
        \langle  \theta^\ast-\hat \theta, \theta^\ast\rangle\le \frac{\mathcal R_n^{\mathcal F}}{n},\text{ where }\hat \theta\triangleq \frac 1n\sum_{i=1}^n x_i.\label{eq:regret-to-sample-complexity}
    \end{equation}
    Hence, if we take $\{\hat \theta_i\}_{i\in S}$ as an estimate of $\{\theta_i^\ast\}_{i\in S}$, the estimation error shrinks as $\mathcal R_n^{\mathcal F}$ is sub-linear and the LHS of \Cref{eq:regret-to-sample-complexity} is non-negative.
    Moreover, as we can show that $\hat \theta$ is not away from $\mathcal X$ by a lot (\Cref{lem:appendix bias term of random projection}), we can safely use $\{\hat \theta_i\}_{i\in S}$ to estimate $\{\theta_i^\ast\}_{i\in S}$ in subsequent phases.
    
    More importantly, if we are granted access to $\mathcal R_n^{\mathcal F}$, we know how close the estimate is; we can proceed to the next stage once it becomes satisfactory.
    But it is unrevealed. Fortunately, we know the regret guarantee of $\mathcal F$, namely $\bar{\mathcal R_n^{\mathcal F}}$, which can serve as a \textit{pessimistic} estimation of ${\mathcal R_n^{\mathcal F}}$. Hence, terminating when $\frac 1n \bar{\mathcal R_n^{\mathcal F}}<\Delta^2$ can ensure $\langle  \theta^\ast-\hat \theta, \theta^\ast\rangle<\Delta^2$ to hold with high probability.
    \item In the ``exploration'' phase, as mentioned before, we can keep the regret incurred by the coordinates identified in $S$ small by putting mass $\hat \theta_i$ for each $i\in S$. For the remaining ones, we use random projection, an idea borrowed from compressed sensing literature \citep{blumensath2009iterative,carpentier2012bandit}, to find those with large magnitudes to add them to $S$.
    
    One may notice that putting mass $\hat \theta_i$ for all $i\in S$ will induce bias to our estimation as $\sum_{i\in S} \hat \theta_i^2\ne \sum_{i\in S} \hat \theta_i\theta_i^\ast$.
    However, as $\hat \theta_i$ is close to $\theta_i^\ast$, this bias will be bounded by $\mathcal O(\Delta^2)$ and become dominated by $\frac \Delta 4$ as $\Delta$ decreases.
    Hence, if we omit this bias, we can overestimate the estimation error due to standard concentration inequalities like Empirical Bernstein \citep{maurer2009empirical,zhang2021improved}.
    Once it becomes small enough, we alternate to the ``commit'' phase again.
    
\end{enumerate}

Therefore, with high probability, we can ensure all coordinates not in $S$ have magnitudes no more than $\mathcal O(\Delta)$ and all coordinates in $S$ will together contribute regret bounded by $\O(\Delta^2)$. Hence, the regret in each step is (roughly) bounded by $\O(s\Delta^2)$. Upper bounding the number of steps needed for each stage and exploiting the regret guarantees of the chosen $\mathcal F$ then gives well-bounded regret.

\subsection{Analysis of the Framework}\label{sec:analysis}

\textbf{Notations.} For each $\Delta$, let $\mathcal T_\Delta$ be the set of rounds associated with $\Delta$. By our algorithm, each $\mathcal T_\Delta$ should be an interval.
Moreover, $\{\mathcal{T}_{\Delta}\}_{\Delta}$ forms a partition of $[T]$. Define $\mathcal T_\Delta^a$ as all the rounds in the ``commit'' phase when the gap threshold is $\Delta$ (where $\mathcal F$ is executed), and $\mathcal T_\Delta^b$ as the ``explore'' phase (i.e., those executing \textsc{RandomProjection}). Let $\tilde{\mathcal T_\Delta^a}$ and $\tilde{\mathcal T_\Delta^b}$ be the steps that the agent decided not to proceed in $\mathcal T_\Delta^a$ and $\mathcal T_\Delta^b$, respectively, which are formally defined as $\tilde{\mathcal T_\Delta^i}=\{t\in \mathcal T_\Delta^i\mid t\ne \max_{t'\in \mathcal T_\Delta^i} t'\}$, $i=a,b$.
Define the final value of $\Delta$ as $\Delta_f$. Denote $n_\Delta^a=\lvert \mathcal T_\Delta^a\rvert$ and $n_\Delta^b=\lvert \mathcal T_\Delta^b\rvert$ (both are stopping times). We have $\sum_{\Delta=2^{-2},\ldots,\Delta_f} (n_\Delta^a+n_\Delta^b)=T$.
We can then decompose $\mathcal R_T$ into $\mathcal R_T^a$ and $\mathcal R_T^b$:
\begin{equation*}
    \mathcal R_T^a=\sum_{\Delta=2^{-2},\ldots,\Delta_f} \sum_{t\in \mathcal T_\Delta^a} \langle {\theta^\ast}-{x_t},{\theta^\ast}\rangle,\quad \mathcal R_T^b=\sum_{\Delta=2^{-2},\ldots,\Delta_f} \sum_{t\in \mathcal T_\Delta^b} \langle {\theta^\ast}-{x_t},{\theta^\ast}\rangle,
\end{equation*}

where $\mathcal R_T^a$ may depend on the choice of $\mathcal F$ and $\mathcal R_T^b$ only depends on the framework (\Cref{alg:framework}) itself. We now show that, as long as the regret estimation $\bar{\mathcal R_n^{\mathcal F}}$ is indeed an overestimation of $\mathcal R_n^{\mathcal F}$ with high probability, we can get a good upper bound of $\mathcal R_T^b$, which is formally stated as \Cref{thm:total regret of part b}. The full proof of \Cref{thm:total regret of part b} will be presented in \Cref{sec:appendix analysis} and is only sketched here.
\begin{theorem}\label{thm:total regret of part b}
Suppose that for any execution of $\mathcal F$ that last for $n$ steps, $\bar{\mathcal R_n^{\mathcal F}}\ge \mathcal R_n^{\mathcal F}$ holds with probability $1-\delta$, i.e., $\bar{\mathcal R_n^{\mathcal F}}$ is pessimistic. Then the total regret incurred by the second phase satisfies
\begin{align*}
    \mathcal R_T^b&= \Otil\left (s\sqrt d\sqrt{\sum_{t=1}^T \sigma_t^2}\log \frac 1\delta+s\log \frac 1\delta\right )\quad \text{with probability }1-\delta.
\end{align*}
\end{theorem}

\textbf{Remark.}
This theorem indicates that our framework \textit{itself} will only induce an $(s\sqrt d,s)$-variance-awareness to the resulting algorithm. As noticed by \citet{abbasi2011improved}, when $\sigma_t\equiv 1$, $\Omega(\sqrt{sdT})$ regret is unavoidable, which means that it is only sub-optimal by a factor no more than $\sqrt s$. Moreover, for deterministic cases, the $\Otil(s)$ regret also matches the aforementioned divide-and-conquer algorithm, which is specially designed and can only work for deterministic cases.

\begin{proof}[Proof Sketch of \Cref{thm:total regret of part b}]
We define two good events with high probability for a given gap threshold $\Delta$: $\mathcal G_\Delta$ and $\mathcal H_\Delta$. Informally, $\mathcal G_\Delta$ means $\sum_{i\in S}\theta_i^\ast(\theta_i^\ast-\hat \theta_i)<\Delta^2$ (i.e., $\hat \theta$ is close to $\theta^\ast$ after ``commit'') and $\mathcal H_\Delta$ stands for $\lvert \theta_i^\ast\rvert\ge \Omega(\Delta)$ if and only if $i\in S$ (i.e., we ``explore'' correctly). Check \Cref{eq:good events} in the appendix for formal definitions. For $\mathcal G_\Delta$, from \Cref{eq:regret-to-sample-complexity}, we know that it happens as long as $\bar{\mathcal R_n^{\mathcal F}}\ge \mathcal R_n^{\mathcal F}$. It remains to argue that $\Pr\{\mathcal H_\Delta\mid \mathcal G_\Delta,\mathcal H_{2\Delta}\}\ge 1-s\delta$.



By \Cref{alg:random projection}, the $i$-th coordinate of each $r_k$ ($1\le k\le n_\Delta^b$) is an independent sample of
\begin{equation}\label{eq:random projection Z_n^k}
    \frac{K}{R^2}\left (y_i\right )^2\theta_i^\ast+\sum_{j\in S}\hat \theta_j(\theta_j^\ast-\hat \theta_j)+\sum_{j\notin S,j\ne i}\left (\frac{K}{R^2}y_iy_j\right )\theta_j^\ast+\left (\frac{K}{R^2}y_i\right )\eta_n,
\end{equation}

where $\frac{\sqrt K}{R}y_i$ is an independent Rademacher random variable. After conditioning on $\mathcal G_\Delta$ and $\mathcal H_{2\Delta}$, $\sum_{i\in S}\hat \theta_i^2$ and $\sum_{i\in S} \hat \theta_i \theta_i^\ast$ will be close. Therefore, the first term is exactly $\theta_i^\ast$ (the magnitude we want to estimate), the second term is a small bias bounded by $\O(\Delta^2)$ and the last two terms are zero-mean noises, which are bounded by $\frac \Delta 4$ according to Empirical Bernstein Inequality (\Cref{thm:Bernstein with common mean}) and our choice of $n_\Delta^b$ (\Cref{eq:designing N in Option II}). Hence, $\Pr\{\mathcal H_\Delta\mid \mathcal G_\Delta,\mathcal H_{2\Delta}\}\ge 1-s\delta$.

Let us focus on an arm $i^\ast$ never identified into $S$ in \Cref{alg:framework}. By definition of $n_\Delta^b$ (\Cref{eq:designing N in Option II}),
\begin{equation*}
    (n_\Delta^b-1)\frac \Delta 4<2\sqrt{2\sum_{t\in \tilde{\mathcal T_\Delta^b}} (r_{t,i^\ast}-\bar r_{i^\ast})^2\ln \frac 4\delta}\le 2\sqrt{2\sum_{t\in \tilde{\mathcal T_\Delta^b}} (r_{t,i^\ast}-\mathbb E[r_{t,i^\ast}])^2\ln \frac 4\delta},
\end{equation*}

where the second inequality is due to properties of sample variances.
By $\mathcal G_\Delta$, those coordinates in $S$ will incur regret of $\sum_{i\in S}(\theta_i^\ast-x_{t,i})\theta_i^\ast=\sum_{i\in S}(\theta_i^\ast-\hat \theta_i)\theta_i^\ast<\Delta^2$ for all $t\in \mathcal T_\Delta^b$. Moreover, by $\mathcal H_{2\Delta}$, each arm outside $S$ will roughly incur $n_\Delta^b(\theta_i^\ast)^2=\O(n_\Delta^b\Delta^2)$ regret, as $y_i$'s are independent and zero-mean. As there are at most $s$ non-zero coordinates, the total regret for $\mathcal T_\Delta^b$ will be roughly bounded by $\O(n_\Delta^b\cdot s\Delta^2)$ (there exists another term due to randomized $y_i$'s, which is dominated and omitted here; check \Cref{lem:appendix single step regret bound} for more details). Hence, the total regret is bounded by
\begin{equation*}
    \mathcal R_T^b\lesssim \sum_{\Delta} \O(sn_\Delta^b\Delta^2)
    =s\cdot \Otil\left (\sum_{\Delta}\Delta \sqrt{\sum_{t\in \mathcal T_\Delta^b} (r_{t,i^\ast}-\E[r_{t,i^\ast}])^2\ln \frac 4\delta}\right )+\O(s).
\end{equation*}

To avoid undesired $\operatorname{poly}(T)$ factors, we cannot directly apply Cauchy-Schwartz inequality to the sum of square roots (as there are a lot of $\Delta$'s). Instead, again by definition of $n_\Delta^b$ (\Cref{eq:designing N in Option II}), we observe the following lower bound of $n_\Delta^b$, which holds for all $\Delta$'s except for $\Delta_f$: $n_\Delta^b \ge \O\big (\frac{1}{\Delta}\sqrt{\sum_{t\in \mathcal T_\Delta^b}(r_{t,i^\ast}-\E[r_{t,i^\ast}])^2\ln \frac 1\delta}\big )$. As $\sum_{\Delta}n_\Delta^b\le T$, some arithmetic calculation gives (intuitively, by thresholding, we 
manage to ``move'' the summation over $\Delta$ into the square root, though suffering an extra logarithmic factor; see \Cref{eq:summation technique} in the appendix for more details)
\begin{equation*}
    \sum_{\Delta\ne \Delta_f}\Delta\sqrt{\sum_{t\in \mathcal T_\Delta^b} (r_{t,i^\ast}-\E[r_{t,i^\ast}])^2}=\Otil\left (\sqrt{\sum_{\Delta\ne \Delta_f}\Delta^2\sum_{t\in \mathcal T_\Delta^b} (r_{t,i^\ast}-\E[r_{t,i^\ast}])^2}\right ).
\end{equation*}

For a given $\Delta$ and any $1\le k\le n_\Delta^b$, the expectation of $(r_{k,i^\ast}-\E[r_{k,i^\ast}])^2$ is bounded by $\left (1+\frac{K}{R^2}\sigma_k^2\right )$ (\Cref{eq:variance of feedback} in the appendix), which is no more than $\left (1+\frac{4d}{\Delta^2}\sigma_k^2\right )$. By concentration properties in the sample variances (\Cref{thm:variance concentration final} in the appendix), the empirical $(r_{k,i^\ast}-\E[r_{k,i^\ast}])^2$ should also be close to $(1+\frac{4}{d^2}\sigma_k^2)$; hence, one can write (omitting all $\log \frac 1\delta$ terms)
\begin{equation*}
    \mathcal R_T^b=\O\left (\sum_{\Delta} sn_\Delta^b\Delta^2\right )=\Otil\left (s\sqrt{\sum_{\Delta} n_\Delta^b\Delta^2}+s\sqrt{d\sum_{t=1}^T \sigma_t^2}\right ).
\end{equation*}
As $\sum_\Delta n_\Delta^b\Delta^2$ appears on both sides, we can apply the ``self-bounding'' property \citep[Lemma 38]{efroni2020exploration} to conclude $\mathcal R_T^b = \O(\sum_{\Delta} sn_\Delta^b\Delta^2) = \Otil\big (s\sqrt{d\sum_{t=1}^T \sigma_t^2}+s\big )$, as claimed.
\end{proof}

\section{Applications of the Proposed Framework}
After showing \Cref{thm:appendix total regret of part b}, it only remains to bound $\mathcal R_T^a$, which depends on the choice of the plug-in algorithm $\mathcal F$. In this section, we give two specific choices of $\mathcal F$, \texttt{VOFUL2} \citep{kim2021improved} and \texttt{Weighted OFUL} \citep{zhou2021nearly}. The former algorithm does not require the information of $\sigma_t$'s (i.e., it works in unknown-variance cases), albeit computationally inefficient. In contrast, the latter is computationally efficient but requires $\sigma_t^2$ to be revealed with the feedback $r_t$ at round $t$.

\subsection{Computationally Inefficient Algorithm for Unknown Variances}\label{sec:zhang et al}

We first use the \texttt{VOFUL2} algorithm from \citet{kim2021improved} as the plug-in algorithm $\mathcal F$, which has the following regret guarantee.
Note that this is slightly stronger than the original bound: We derive a strengthened ``self-bounding'' version of it (the first inequality), which is critical to our analysis.
\begin{proposition}[{\citet[Variant of Theorem 2]{kim2021improved}}]\label{prop:VOFUL2 regret}
\texttt{VOFUL2} executed for $n$ rounds on $d$ dimensions guarantees, w.p. at least $1-\delta$, there exists a constant $C=\Otil(1)$ such that $\mathcal R_n^{\mathcal F}\le C\big (d^{1.5}\sqrt{\sum_{k=1}^n\eta_k^2\ln \frac 1\delta}+d^2\ln \frac 4\delta\big )=\Otil\big (d^{1.5}\sqrt{\sum_{k=1}^n\sigma_k^2}\log \frac 1\delta+d^2\log \frac 4\delta\big )$,
where $n$ is a stopping time finite a.s. and $\sigma_1^2,\sigma_2^2,\ldots,\sigma_n^2$ are the variances of the independent Gaussians $\eta_1,\eta_2,\ldots,\eta_n$.
\end{proposition}

We now construct the regret over-estimation $\bar{\mathcal R_n^{\mathcal F}}$. Due to unknown variances, it is not straightforward. Our rescue is to use ridge linear regression $\hat \beta\triangleq \operatornamewithlimits{argmin}_{ \beta\in \mathbb R^d}\left (\sum_{k=1}^n (r_k-\langle x_k, \beta\rangle)^2+\lambda \lVert  \beta\rVert_2\right )$ for samples $\{(x_k,r_k)\}_{k=1}^n$, which
ensures that the empirical variance estimation $\sum_{k=1}^n (r_k-\langle x_k, \hat \beta\rangle)^2$ differs from the true sample variance $\sum_{k=1}^n \eta_k^2=\sum_{k=1}^n (r_k-\langle  x_k, \beta^\ast\rangle)^2$ by no more than $\Otil(s\log \frac 1\delta)$ (check \Cref{sec:ridge error bound} for a formal version). Accordingly, from \Cref{prop:VOFUL2 regret}, we can see that
{\small\begin{equation}\label{eq:regret estimation with unknown variance main text}
    \mathcal R_n^{\mathcal F}\le \bar{\mathcal R_n^{\mathcal F}}\triangleq C\left (s^{1.5}\sqrt{\sum_{k=1}^n (r_k-\langle  x_k, \hat \beta\rangle)^2\ln \frac 1\delta}+s^2\sqrt{2\ln \frac{n}{s\delta^2}\ln \frac 1\delta}+s^{1.5}\sqrt{2\ln \frac 1\delta}+s^2\ln \frac 1\delta\right ).
\end{equation}}

Moreover, one can observe that the total sample variance $\sum_{k=1}^n \eta_k^2$ is bounded by (a constant multiple of) the total variance $\sum_{k=1}^n \sigma_k^2$ (which is formally stated as \Cref{thm:variance concentration raw} in the appendix). Therefore, with \Cref{eq:regret estimation with unknown variance main text} as our pessimistic regret estimation $\bar{\mathcal R_n^{\mathcal F}}$, we have the following regret guarantee.
\begin{theorem}[Regret of \Cref{alg:framework} with \texttt{VOFUL2}]\label{thm:regret with VOFUL2 and unknown variances}
\Cref{alg:framework} with \texttt{VOFUL2} as $\mathcal F$ and $\bar{\mathcal R_n^{\mathcal F}}$ defined in \Cref{eq:regret estimation with unknown variance main text} ensures that $\mathcal R_T= \tilde{\mathcal O}\big ((s^{2.5}+s\sqrt d)\sqrt{\sum_{t=1}^T \sigma_t^2} \log \frac 1\delta+s^3\log \frac 1\delta\big )$ with probability $1-\delta$.
\end{theorem}

Due to space limitations, we defer the full proof to \Cref{sec:appendix zhang et al unknown} and only sketch it here.

\begin{proof}[Proof Sketch of \Cref{thm:regret with VOFUL2 and unknown variances}]
To bound $\mathcal R_T^a$, we consider the regret from the coordinates in and outside $S$ separately. For the former, the total regret in a single phase with gap threshold $\Delta$ is simply controlled by $\Otil \big (s^{1.5}\sqrt{\sum_{t\in \mathcal T_\Delta^a}\eta_t^2\log \frac 1\delta}+s^2\log \frac 1\delta \big )$ (thanks to \Cref{prop:VOFUL2 regret}). For the latter, each non-zero coordinate outside $S$ can at most incur $\O(\Delta^2)$ regret for each $t\in \mathcal T_\Delta^a$. By definition of $n_\Delta^a$ (\Cref{line:phase a terminiate}), we have $n_\Delta^a=\lvert \mathcal T_\Delta^a\rvert=\O\big (\frac{s^{1.5}}{\Delta^2} \sqrt{\sum_{t\in \tilde{\mathcal T_\Delta^a}} \eta_t^2\ln \frac 1\delta}+\frac{s^2}{\Delta^2}\ln \frac 1\delta\big )$
, just like the proof of \Cref{thm:total regret of part b}.
As the regret from the second part is bounded by $\O(s\Delta^2\cdot n_\Delta^a)$, these two parts together sum to
\begin{equation*}
    \mathcal R_T^a\le \sum_{\Delta} \O \left (s^{2.5}\sqrt{\sum_{t\in {\mathcal T_\Delta^a}} \eta_t^2\log \frac 1\delta}+s^3\log \frac 1\delta+s\Delta^2\right ).
\end{equation*}

As in \Cref{thm:total regret of part b}, we notice that $n_\Delta^a=\Omega\big (\frac{s^{1.5}}{\Delta^2} \sqrt{\sum_{t\in \tilde{\mathcal T_\Delta^a}} \eta_t^2\ln \frac 1\delta}+\frac{s^2}{\Delta^2}\ln \frac 1\delta\big )$ for all $\Delta\ne \Delta_f$ again by definition of $n_\Delta^a$. This will move the summation over $\Delta$ into the square root. Moreover, by the fact that $\eta_t^2=\O(\sigma_t^2\log \frac 1\delta)$ (\Cref{thm:variance concentration final} in the appendix), we have $\mathcal R_T^a=\Otil\big (s^{2.5}\sqrt{\sum_{t=1}^T \sigma_t^2}\log \frac 1\delta+s^3\log \frac 1\delta\big )$.
Combining this with the bound of $\mathcal R_T^b$ provided by \Cref{thm:total regret of part b} concludes the proof.
\end{proof}

\subsection{Computationally Efficient Algorithm for Known Variances}\label{sec:zhou et al}

In this section, we consider a computational efficient algorithm \texttt{Weighted OFUL} \citep{zhou2021nearly}, which itself requires $\sigma_t^2$ to be presented at the end of round $t$. Their algorithm guarantees:
\begin{proposition}[{\citet[Corollary 4.3]{zhou2021nearly}}]\label{prop:Weighted OFUL regret}
With probability at least $1-\delta$, \texttt{Weighted OFUL} executed for $n$ steps on $d$ dimensions guarantees $\mathcal R_T^{\mathcal F}\le C(\sqrt{dn\log \frac 1\delta}+d\sqrt{\sum_{k=1}^n \sigma_k^2\log \frac 1\delta})$, where $C=\Otil(1)$, $n$ is a stopping time finite a.s., and $\sigma_1^2,\sigma_2^2,\ldots,\sigma_n^2$ are the variances of $\eta_1,\eta_2,\ldots,\eta_n$.
\end{proposition}

Taking $\mathcal F$ as \texttt{Weighted OFUL}, we will have the following regret guarantee for sparse linear bandits:
\begin{theorem}[Regret of \Cref{alg:framework} with \texttt{Weighted OFUL}]\label{thm:regret with Weighted OUFL}
\Cref{alg:framework} with \texttt{Weighted OFUL} as $\mathcal F$ and $\bar{\mathcal R_n^{\mathcal F}}$ defined as $C\big (\sqrt{sn\ln \frac 1\delta}+s\sqrt{\sum_{k=1}^n \sigma_k^2\ln \frac 1\delta}\big )$ guarantees $\mathcal R_T= \Otil\big ((s^2+s\sqrt d )\sqrt{\sum_{t=1}^T \sigma_t^2}\log \frac 1\delta+s^{1.5}\sqrt{T}\log \frac 1\delta\big )$ with probability $1-\delta$.
\end{theorem}

The proof is similar to that of \Cref{thm:regret with VOFUL2 and unknown variances}, i.e., bounding $n_\Delta^a$ by \Cref{line:phase a terminiate} of \Cref{alg:framework} and then using summation techniques to move the summation over $\Delta$ into the square root. The only difference is that we will need to bound $\O (\sum_{\Delta} \Delta^{-2})$, which seems to be as large as $T$ if we follow the analysis of \Cref{thm:regret with VOFUL2 and unknown variances}. However, as we included an additive factor $\sqrt{sn\ln \frac 1\delta}$ in the regret over-estimation $\bar{\mathcal R_n^{\mathcal F}}$, we have $n_\Delta^a\ge \Delta^{-2}\sqrt{sn_\Delta^a \ln \frac 1\delta}$, which means $n_\Delta^a=\Omega(s\Delta^{-4})$. From $\sum_\Delta n_\Delta^a\le T$, we can consequently bound $\sum_\Delta \Delta^{-2}$ as $\O(\sqrt{\frac Ts})$. The remaining part is just an analog of \Cref{thm:regret with VOFUL2 and unknown variances}. Therefore, the proof is omitted in the main text and postponed to \Cref{sec:appendix zhou et al}.

\section{Conclusion}

We considered the sparse linear bandit problem with heteroscedastic noises and provided a general framework to reduce \textit{any} variance-aware linear bandit algorithm $\mathcal F$ to an algorithm $\mathcal G$ for sparse linear bandits that is also variance-aware.
We first applied the computationally inefficient algorithm \texttt{VOFUL} from \citet{zhang2021improved} and \citet{kim2021improved}. The resulting algorithm works for the unknown-variance case and gets $\Otil((s^{2.5}+s\sqrt d)\sqrt{\sum_{t=1}^T\sigma_t^2}\log \frac 1\delta+s^3\log \frac 1\delta)$ regret, which, when regarding the sparsity factor $s\ll d$ as a constant, not only is worst-case optimal but also enjoys constant regret in deterministic cases. We also applied the efficient algorithm \texttt{Weighted OFUL} by \citet{zhou2021nearly} that requires known variance; we got $\Otil((s^{2}+s\sqrt d)\sqrt{\sum_{t=1}^T \sigma_t^2}\log \frac 1\delta+(\sqrt{sT}+s)\log \frac 1\delta)$ regret, still independent of $d$ in deterministic cases. See \Cref{sec:future directions} for several future directions.

\ificlrfinal
\subsubsection*{Acknowledgments}
We greatly acknowledge Csaba Szepesv\'ari for sharing the manuscript \citep{antos09sparse} with us, which shows the $\Omega(\sqrt{dT})$ lower bound for sparse linear bandits when $s=1$ and the action sets are unit balls.
The authors thank Xiequan Fan from Tianjin University for the constructive discussion about the application of \Cref{thm:de la pena raw,thm:generalized Freedman}.
At last, we thank the anonymous reviewers for their detailed reviews, from which we benefit greatly.
This work was supported in part by NSF CCF 2212261, NSF IIS 2143493, NSF DMS-2134106, NSF CCF
2019844 and NSF IIS 2110170.
\fi

\bibliography{references}

\begin{thebibliography}{43}
\providecommand{\natexlab}[1]{#1}
\providecommand{\url}[1]{\texttt{#1}}
\expandafter\ifx\csname urlstyle\endcsname\relax
  \providecommand{\doi}[1]{doi: #1}\else
  \providecommand{\doi}{doi: \begingroup \urlstyle{rm}\Url}\fi

\bibitem[Abbasi-Yadkori et~al.(2011)Abbasi-Yadkori, P{\'a}l, and
  Szepesv{\'a}ri]{abbasi2011improved}
Yasin Abbasi-Yadkori, D{\'a}vid P{\'a}l, and Csaba Szepesv{\'a}ri.
\newblock Improved algorithms for linear stochastic bandits.
\newblock \emph{Advances in neural information processing systems}, 24, 2011.

\bibitem[Abbasi-Yadkori et~al.(2012)Abbasi-Yadkori, Pal, and
  Szepesvari]{abbasi2012online}
Yasin Abbasi-Yadkori, David Pal, and Csaba Szepesvari.
\newblock Online-to-confidence-set conversions and application to sparse
  stochastic bandits.
\newblock In \emph{Artificial Intelligence and Statistics}, pp.\  1--9. PMLR,
  2012.

\bibitem[Alieva et~al.(2021)Alieva, Cutkosky, and Das]{alieva2021robust}
Ayya Alieva, Ashok Cutkosky, and Abhimanyu Das.
\newblock Robust pure exploration in linear bandits with limited budget.
\newblock In \emph{International Conference on Machine Learning}, pp.\
  187--195. PMLR, 2021.

\bibitem[Antos \& Szepesv{\'a}ri(2009)Antos and Szepesv{\'a}ri]{antos09sparse}
Andr{\'a}s Antos and Csaba Szepesv{\'a}ri.
\newblock Stochastic bandits with large action sets revisited.
\newblock Personal communication, 2009.

\bibitem[Ariu et~al.(2022)Ariu, Abe, and Prouti{\`e}re]{ariu2022thresholded}
Kaito Ariu, Kenshi Abe, and Alexandre Prouti{\`e}re.
\newblock Thresholded lasso bandit.
\newblock In \emph{International Conference on Machine Learning}, pp.\
  878--928. PMLR, 2022.

\bibitem[Audibert et~al.(2009)Audibert, Munos, and
  Szepesv{\'a}ri]{audibert2009exploration}
Jean-Yves Audibert, R{\'e}mi Munos, and Csaba Szepesv{\'a}ri.
\newblock Exploration--exploitation tradeoff using variance estimates in
  multi-armed bandits.
\newblock \emph{Theoretical Computer Science}, 410\penalty0 (19):\penalty0
  1876--1902, 2009.

\bibitem[Azar et~al.(2017)Azar, Osband, and Munos]{azar2017minimax}
Mohammad~Gheshlaghi Azar, Ian Osband, and R{\'e}mi Munos.
\newblock Minimax regret bounds for reinforcement learning.
\newblock In \emph{International Conference on Machine Learning}, pp.\
  263--272. PMLR, 2017.

\bibitem[Bastani \& Bayati(2020)Bastani and Bayati]{bastani2020online}
Hamsa Bastani and Mohsen Bayati.
\newblock Online decision making with high-dimensional covariates.
\newblock \emph{Operations Research}, 68\penalty0 (1):\penalty0 276--294, 2020.

\bibitem[Blumensath \& Davies(2009)Blumensath and
  Davies]{blumensath2009iterative}
Thomas Blumensath and Mike~E Davies.
\newblock Iterative hard thresholding for compressed sensing.
\newblock \emph{Applied and computational harmonic analysis}, 27\penalty0
  (3):\penalty0 265--274, 2009.

\bibitem[Carpentier \& Munos(2012)Carpentier and Munos]{carpentier2012bandit}
Alexandra Carpentier and R{\'e}mi Munos.
\newblock Bandit theory meets compressed sensing for high dimensional
  stochastic linear bandit.
\newblock In \emph{Artificial Intelligence and Statistics}, pp.\  190--198.
  PMLR, 2012.

\bibitem[Chu et~al.(2011)Chu, Li, Reyzin, and Schapire]{chu2011contextual}
Wei Chu, Lihong Li, Lev Reyzin, and Robert Schapire.
\newblock Contextual bandits with linear payoff functions.
\newblock In \emph{Proceedings of the Fourteenth International Conference on
  Artificial Intelligence and Statistics}, pp.\  208--214. JMLR Workshop and
  Conference Proceedings, 2011.

\bibitem[Dani et~al.(2008)Dani, Hayes, and Kakade]{dani2008stochastic}
Varsha Dani, Thomas~P Hayes, and Sham~M Kakade.
\newblock Stochastic linear optimization under bandit feedback.
\newblock In \emph{21st Annual Conference on Learning Theory}, pp.\  355--366.
  Omnipress, 2008.

\bibitem[Degenne et~al.(2019)Degenne, Koolen, and M{\'e}nard]{degenne2019non}
R{\'e}my Degenne, Wouter~M Koolen, and Pierre M{\'e}nard.
\newblock Non-asymptotic pure exploration by solving games.
\newblock \emph{Advances in Neural Information Processing Systems}, 32, 2019.

\bibitem[Dong et~al.(2021)Dong, Yang, and Ma]{dong2021provable}
Kefan Dong, Jiaqi Yang, and Tengyu Ma.
\newblock Provable model-based nonlinear bandit and reinforcement learning:
  Shelve optimism, embrace virtual curvature.
\newblock \emph{Advances in Neural Information Processing Systems}, 34, 2021.

\bibitem[Efroni et~al.(2020)Efroni, Mannor, and Pirotta]{efroni2020exploration}
Yonathan Efroni, Shie Mannor, and Matteo Pirotta.
\newblock Exploration-exploitation in constrained mdps.
\newblock \emph{arXiv preprint arXiv:2003.02189}, 2020.

\bibitem[Fan et~al.(2015)Fan, Grama, and Liu]{fan2015exponential}
Xiequan Fan, Ion Grama, and Quansheng Liu.
\newblock Exponential inequalities for martingales with applications.
\newblock \emph{Electronic Journal of Probability}, 20:\penalty0 1--22, 2015.

\bibitem[Faury et~al.(2020)Faury, Abeille, Calauz{\`e}nes, and
  Fercoq]{faury2020improved}
Louis Faury, Marc Abeille, Cl{\'e}ment Calauz{\`e}nes, and Olivier Fercoq.
\newblock Improved optimistic algorithms for logistic bandits.
\newblock In \emph{International Conference on Machine Learning}, pp.\
  3052--3060. PMLR, 2020.

\bibitem[Freedman(1975)]{freedman1975tail}
David~A Freedman.
\newblock On tail probabilities for martingales.
\newblock \emph{the Annals of Probability}, pp.\  100--118, 1975.

\bibitem[Hao et~al.(2020)Hao, Lattimore, and Wang]{hao2020high}
Botao Hao, Tor Lattimore, and Mengdi Wang.
\newblock High-dimensional sparse linear bandits.
\newblock \emph{Advances in Neural Information Processing Systems},
  33:\penalty0 10753--10763, 2020.

\bibitem[Hao et~al.(2021{\natexlab{a}})Hao, Lattimore, and
  Deng]{hao2021information}
Botao Hao, Tor Lattimore, and Wei Deng.
\newblock Information directed sampling for sparse linear bandits.
\newblock \emph{Advances in Neural Information Processing Systems}, 34,
  2021{\natexlab{a}}.

\bibitem[Hao et~al.(2021{\natexlab{b}})Hao, Lattimore, Szepesv{\'a}ri, and
  Wang]{hao2021online}
Botao Hao, Tor Lattimore, Csaba Szepesv{\'a}ri, and Mengdi Wang.
\newblock Online sparse reinforcement learning.
\newblock In \emph{International Conference on Artificial Intelligence and
  Statistics}, pp.\  316--324. PMLR, 2021{\natexlab{b}}.

\bibitem[Honorio \& Jaakkola(2014)Honorio and Jaakkola]{honorio2014tight}
Jean Honorio and Tommi Jaakkola.
\newblock Tight bounds for the expected risk of linear classifiers and
  pac-bayes finite-sample guarantees.
\newblock In \emph{Artificial Intelligence and Statistics}, pp.\  384--392.
  PMLR, 2014.

\bibitem[Hou et~al.(2022)Hou, Tan, and Zhong]{hou2022almost}
Yunlong Hou, Vincent~YF Tan, and Zixin Zhong.
\newblock Almost optimal variance-constrained best arm identification.
\newblock \emph{arXiv preprint arXiv:2201.10142}, 2022.

\bibitem[Jedra \& Proutiere(2020)Jedra and Proutiere]{jedra2020optimal}
Yassir Jedra and Alexandre Proutiere.
\newblock Optimal best-arm identification in linear bandits.
\newblock \emph{Advances in Neural Information Processing Systems},
  33:\penalty0 10007--10017, 2020.

\bibitem[Jin et~al.(2018)Jin, Allen-Zhu, Bubeck, and Jordan]{jin2018q}
Chi Jin, Zeyuan Allen-Zhu, Sebastien Bubeck, and Michael~I Jordan.
\newblock Is q-learning provably efficient?
\newblock \emph{Advances in neural information processing systems}, 31, 2018.

\bibitem[Kannan et~al.(2018)Kannan, Morgenstern, Roth, Waggoner, and
  Wu]{kannan2018smoothed}
Sampath Kannan, Jamie~H Morgenstern, Aaron Roth, Bo~Waggoner, and Zhiwei~Steven
  Wu.
\newblock A smoothed analysis of the greedy algorithm for the linear contextual
  bandit problem.
\newblock \emph{Advances in neural information processing systems}, 31, 2018.

\bibitem[Kim \& Paik(2019)Kim and Paik]{kim2019doubly}
Gi-Soo Kim and Myunghee~Cho Paik.
\newblock Doubly-robust lasso bandit.
\newblock \emph{Advances in Neural Information Processing Systems}, 32, 2019.

\bibitem[Kim et~al.(2021)Kim, Yang, and Jun]{kim2021improved}
Yeoneung Kim, Insoon Yang, and Kwang-Sung Jun.
\newblock Improved regret analysis for variance-adaptive linear bandits and
  horizon-free linear mixture mdps.
\newblock \emph{arXiv preprint arXiv:2111.03289}, 2021.

\bibitem[Kirschner \& Krause(2018)Kirschner and
  Krause]{kirschner2018information}
Johannes Kirschner and Andreas Krause.
\newblock Information directed sampling and bandits with heteroscedastic noise.
\newblock In \emph{Conference On Learning Theory}, pp.\  358--384. PMLR, 2018.

\bibitem[Lattimore \& Hutter(2012)Lattimore and Hutter]{lattimore2012pac}
Tor Lattimore and Marcus Hutter.
\newblock Pac bounds for discounted mdps.
\newblock In \emph{International Conference on Algorithmic Learning Theory},
  pp.\  320--334. Springer, 2012.

\bibitem[Lattimore \& Szepesv{\'a}ri(2020)Lattimore and
  Szepesv{\'a}ri]{lattimore2020bandit}
Tor Lattimore and Csaba Szepesv{\'a}ri.
\newblock \emph{Bandit algorithms}.
\newblock Cambridge University Press, 2020.

\bibitem[Lattimore et~al.(2015)Lattimore, Crammer, and
  Szepesv{\'a}ri]{lattimore2015linear}
Tor Lattimore, Koby Crammer, and Csaba Szepesv{\'a}ri.
\newblock Linear multi-resource allocation with semi-bandit feedback.
\newblock \emph{Advances in Neural Information Processing Systems}, 28, 2015.

\bibitem[Li et~al.(2019)Li, Wang, and Zhou]{li2019nearly}
Yingkai Li, Yining Wang, and Yuan Zhou.
\newblock Nearly minimax-optimal regret for linearly parameterized bandits.
\newblock In \emph{Conference on Learning Theory}, pp.\  2173--2174. PMLR,
  2019.

\bibitem[Li et~al.(2021)Li, Wang, Chen, and Zhou]{li2021tight}
Yingkai Li, Yining Wang, Xi~Chen, and Yuan Zhou.
\newblock Tight regret bounds for infinite-armed linear contextual bandits.
\newblock In \emph{International Conference on Artificial Intelligence and
  Statistics}, pp.\  370--378. PMLR, 2021.

\bibitem[Maurer \& Pontil(2009)Maurer and Pontil]{maurer2009empirical}
Andreas Maurer and Massimiliano Pontil.
\newblock Empirical bernstein bounds and sample-variance penalization.
\newblock In \emph{{COLT} 2009 - The 22nd Conference on Learning Theory}, 2009.

\bibitem[Oh et~al.(2021)Oh, Iyengar, and Zeevi]{oh2021sparsity}
Min-hwan Oh, Garud Iyengar, and Assaf Zeevi.
\newblock Sparsity-agnostic lasso bandit.
\newblock In \emph{International Conference on Machine Learning}, pp.\
  8271--8280. PMLR, 2021.

\bibitem[Ren \& Zhou(2020)Ren and Zhou]{ren2020dynamic}
Zhimei Ren and Zhengyuan Zhou.
\newblock Dynamic batch learning in high-dimensional sparse linear contextual
  bandits.
\newblock \emph{arXiv preprint arXiv:2008.11918}, 2020.

\bibitem[Soare et~al.(2014)Soare, Lazaric, and Munos]{soare2014best}
Marta Soare, Alessandro Lazaric, and R{\'e}mi Munos.
\newblock Best-arm identification in linear bandits.
\newblock \emph{Advances in Neural Information Processing Systems}, 27, 2014.

\bibitem[Wang et~al.(2020)Wang, Chen, Fang, Wang, and Li]{wang2020nearly}
Yining Wang, Yi~Chen, Ethan~X Fang, Zhaoran Wang, and Runze Li.
\newblock Nearly dimension-independent sparse linear bandit over small action
  spaces via best subset selection.
\newblock \emph{arXiv preprint arXiv:2009.02003}, 2020.

\bibitem[Zanette \& Brunskill(2019)Zanette and Brunskill]{zanette2019tighter}
Andrea Zanette and Emma Brunskill.
\newblock Tighter problem-dependent regret bounds in reinforcement learning
  without domain knowledge using value function bounds.
\newblock In \emph{International Conference on Machine Learning}, pp.\
  7304--7312. PMLR, 2019.

\bibitem[Zhang et~al.(2021)Zhang, Yang, Ji, and Du]{zhang2021improved}
Zihan Zhang, Jiaqi Yang, Xiangyang Ji, and Simon~S Du.
\newblock Improved variance-aware confidence sets for linear bandits and linear
  mixture mdp.
\newblock \emph{Advances in Neural Information Processing Systems}, 34, 2021.

\bibitem[Zhao et~al.(2022)Zhao, Zhou, He, and Gu]{zhao2022bandit}
Heyang Zhao, Dongruo Zhou, Jiafan He, and Quanquan Gu.
\newblock Bandit learning with general function classes: Heteroscedastic noise
  and variance-dependent regret bounds.
\newblock \emph{arXiv preprint arXiv:2202.13603}, 2022.

\bibitem[Zhou et~al.(2021)Zhou, Gu, and Szepesvari]{zhou2021nearly}
Dongruo Zhou, Quanquan Gu, and Csaba Szepesvari.
\newblock Nearly minimax optimal reinforcement learning for linear mixture
  markov decision processes.
\newblock In \emph{Conference on Learning Theory}, pp.\  4532--4576. PMLR,
  2021.

\end{thebibliography}
\bibliographystyle{iclr2023_conference}


\newpage
\appendix
\renewcommand{\appendixpagename}{\centering \LARGE Supplementary Materials}
\appendixpage

\startcontents[section]
\printcontents[section]{l}{1}{\setcounter{tocdepth}{2}}

\section{More on Related Works}\label{sec:appendix related work}

In this section, we briefly compare to several related works on sparse linear bandits in terms of regret guarantees, noise assumptions and query models.

\begin{itemize}[leftmargin=*]
\item The regret of \citet{abbasi2012online} is $\Otil(Rs\sqrt{dT})$ when assuming conditionally $R$-sub-Gaussian noises (i.e., $\eta_t\mid \mathcal F_{t-1}\sim \text{subG}(R^2)$, which will be formally defined in \Cref{assump:sub-Gaussian}). At the same time, they allow an arbitrary varying action set $D_1,D_2,\ldots,D_T\subseteq \mathbb B^d$ (though they in fact allows arbitrary decision sets $D_1,D_2,\ldots,D_T\subseteq \mathbb R^d$, their regret bound scales with $\max_{x\in D_t}\lVert x\rVert_2$, so we assume $D_t\subseteq \mathbb B^d$ without loss of generality). 
This model is less strictive than ours, as we only allow $D_1=D_2=\cdots=D_T=\mathbb B^d$ (as explained in \Cref{footnote:action set} in the main text.
When the noises are Gaussian with variance $1$ and the ground-truth $\theta^\ast$ is one-hot (i.e., $s=1$), their regret bound reduces to $\Otil(\sqrt{dT})$, which matches the $\Omega(\sqrt{dT})$ bound in \citet{antos09sparse} when the actions sets are allowed to be the entire unit ball (which means the agent will be more powerful than that of \citet{abbasi2012online}).

\item The regret of \citet{carpentier2012bandit} is $\Otil((\lVert \theta\rVert_2+\lVert \sigma\rVert_2)s\sqrt T)$, assuming a unit-ball action set, a $\lVert \theta^\ast\rVert_2\le \lVert \theta\rVert_2$ ground-truth and $\lVert \eta_t\rVert_2\le \lVert \sigma\rVert_2$ noises, where $\eta_t$ will be defined later.

This bound seems to bypass the $\Omega(\sqrt{dT})$ lower bound when $s=1$. However, this is due to a different noise model: They assumed the noise is component-wise, i.e., $r_t=\langle {\theta^\ast}+{\eta_t},{x_t}\rangle$ where $\eta_t\in \mathbb R^d$. In contrast, our model assumed a $\langle \theta^\ast,x_t\rangle+\eta_t$ noise model where $\eta_t\in \mathbb R$. Therefore, the $\max_t \lVert \eta_t\rVert_2$ dependency can be of order $\O(\sqrt d)$ to ensure a similar noise model as ours.

\item The regret of \citet[Appendix G]{lattimore2015linear} is also of order $\Otil(s\sqrt T)$, assuming a $[-1,1]$-bounded noises, a hypercube $\mathcal X=[-1,1]^d$ action set and a $\lVert \theta^\ast\rVert_1\le 1$ ground-truth. We will then explain why this does not violate the $\Omega(\sqrt{sdT})$ regret lower bound as well.

Consider an extreme query $(1,1,\ldots,1)\in \mathcal X$, which is valid in their query model (in fact, their algorithm is a random projection procedure with some carefully designed regularity conditions, so this type of queries appears all the time). However, in our query model where the action set is the unit ball $\mathbb B^d$, we have to scale it by $\frac{1}{\sqrt d}$. As the noise will never be scaled, this will amplify the noises by $\sqrt d$, so we will need $\text{poly}(d)$ more times of queries to get the same confidence set, making the regret bound have a polynomial dependency on $d$.

Moreover, their ground-truth $\theta^\ast$ needs to satisfy $\lVert \theta^\ast\rVert_1\le 1$. However, an $s$-sparse ground-truth $\theta^\ast$ with $2$-norm $1$ can have $1$-norm as much as $\sqrt s$. Therefore, another $\sqrt s$ should also be multiplied for a fair comparison with our algorithm.
\end{itemize}

In conclusion, the second and the third work assumed different noise or query models to amplify the signal-to-noise ratio and thus avoid a polynomial dependency on $d$, compared to the regret bounds of \citet{abbasi2012online} and ours.

However, we have to admit that \citet{abbasi2011improved} allows a drifting action set, whereas ours only allow a unit-sphere action set, just like \citet{carpentier2012bandit}. The reason is discussed in \Cref{footnote:action set} in the main text.

\section{Future Directions}\label{sec:future directions}
First of all, there is still a gap in the worst-case regret in terms of $s$, as the lower bound for sparse linear bandits is $\Omega(\sqrt{sdT})$ instead of our $\Otil( s\sqrt{dT})$ when $\sigma_t\equiv 1$. Closing this gap in $s$ is an interesting future work. Our current algorithm, unfortunately, is incapable of a $\O(\sqrt{sdT})$-style worst-case regret guarantee: Suppose that $T=ds^2$, $\theta_i^\ast=s^{-1/2}$ for $i=1,2,\ldots,s$ (so $\lVert \theta^\ast\rVert_2\le 1$), and $\sigma_t\equiv 1$. Then we have $n_\Delta^b\approx \Delta^{-1} \sqrt{n_\Delta^b(1+d\Delta^{-2})}$, which gives $n_\Delta^b\approx d\Delta^{-4}$. Hence, the total regret will be $\sum_{i=1}^s \sum_{\Delta\ge \theta_i^\ast}n_\Delta^b (\theta_i^\ast)^2\approx d\sum_{i=1}^s (\theta_i^\ast)^2=ds^2=\O(s\sqrt{dT})$. Thus, algorithmic improvements must be made to better dependency on $s$. We leave this for future research.

Moreover, the current work relies on the random projection procedure \citep{carpentier2012bandit}, which only works when the action set is the \textit{unit sphere}. Such an assumption is unrealistic in practice. We wonder whether there is an alternative that only requires a looser condition.

At last, deriving a variance-aware lower bound rather than a minimax one is also important, as it can better illustrate the inherent hardness of the problem with different noise levels. We remark that extending the proof of current minimax lower bounds (see, e.g., \citep{antos09sparse}) to variance-aware ones is not straightforward.

\section{Divide-and-Conquer Algorithm for Deterministic Settings}\label{sec:appendix divide and conquer}
In this section, we discuss how to solve the deterministic sparse linear bandit problem in $\Otil(s)$ steps using a divide-and-conquer algorithm, as we briefly mentioned in the main text.

We mainly adopt the idea mentioned by \citet[Footnote 6]{dong2021provable}. For each divide-and-conquer subroutine working on several coordinates $i_1,i_2,\ldots,i_k\in [d]$, we query half of them (e.g., $i_1,i_2,\ldots,i_{k/2}$ when assuming $2\mid k$) with $\sqrt{\nicefrac 2k}$ mass on each coordinate. This will reveal whether there is a non-zero coordinate among them. If the feedback is non-zero, we then conclude that there exists a non-zero coordinate in this half. Hence, we dive into this half and conquer this sub-problem (i.e., divide-and-conquer). Otherwise, we simply discard this half and consider the other half.

However, this vanilla algorithm proposed by \citet{dong2021provable} fails to consider the possibility that two coordinates cancel each other (e.g., two coordinates with magnitude $\pm \sqrt {\nicefrac 12}$ will make the feedback equal to zero). Fortunately, this problem can be resolved via randomly putting magnitude $\pm \sqrt{\nicefrac 2k}$ on each coordinate, which is similar to the idea illustrated in \Cref{alg:random projection}. As the environment is deterministic, each step will give the correct feedback with probability $1$. Therefore, a constant number of trials is enough to tell whether there exists a non-zero coordinate.

At last, we analyze the number of interactions needed for this approach. As it is guaranteed that each divide-and-conquer subrountine will be working on a set of coordinates where at least one of them is non-zero, we can bound the number of interactions as $\O\left (\sum_{\theta_i^\ast\ne 0} \#\text{subrountines containing }i\right )$.
As for each time we will divide the coordinates into half, there can be at most $\log_2 d$ subrountines containing $i$ for each individual $i$. Therefore, the number of interactions will be $\O(s\log d)$.

After that, we will be sure to find out all coordinates with non-zero magnitudes. Asking each of them once then reveals their actual magnitude. Therefore, we can recover $\theta^\ast$ in $\O(s\log d+s)=\O(s\log d)$ rounds and will not suffer any regret after that. So the regret of this algorithm will indeed be $\Otil(s)$, which is (nearly) independent of $d$ and $T$.

\section{Concentration Inequalities}




\subsection{Sample Mean Upper Bound}
We shall make use of the following self-normalizing result.
\begin{proposition}[{\citet[Remark 2.9]{fan2015exponential}}]\label{thm:de la pena raw}
Suppose that $\{\xi_i\}_{i=1}^n$ are independent and symmetric. Then for all $x>0$,
\begin{equation*}
    \Pr\left \{\max_{1\le k\le n}\frac{\sum_{i=1}^k \xi_i}{\sqrt{\sum_{i=1}^n \xi_i^2}}\ge x\right \}\le \exp\left (-\frac{x^2}{2}\right ).
\end{equation*}
\end{proposition}
\begin{corollary}\label{thm:de la pena}
Let $X_1,X_2,\ldots,X_n$ be a sequence of independent and symmetric random variables where $n$ is a stopping time that is finite a.s. Then for any $\delta>0$, with probability $1-\delta$, we have
\begin{equation*}
    \left \lvert \sum_{i=1}^n (X_i-\mu_i)\right \rvert\le \sqrt{\sum_{i=1}^n (X_i-\mu_i)^2\ln \frac 2\delta}.
\end{equation*}
\end{corollary}
\begin{proof}
This immediately follows by picking $x=2\sqrt{\ln \frac 2\delta}$ and then applying Fatou's lemma.
\end{proof}

Therefore, we can present our Empirical Bernstein Inequality for conditional symmetric stochastic processes with a common mean, as follows:
\begin{theorem}[Empirical Bernstein Inequality]\label{thm:Bernstein with common mean}
For a sequence of independent and symmetric random variables $X_1,X_2,\ldots,X_n$ that shares a common mean (i.e., $\E[X_i]=\mu$ for some $\mu$ for all $i$), we have the following inequality where $n$ is a stopping time finite a.s.
\begin{equation*}
    \Pr\left \{\left \lvert \sum_{i=1}^n (X_i-\mu)\right \rvert\le \sqrt{2\sum_{i=1}^n (X_i-\bar X)^2\ln \frac 4\delta}\right \}\ge 1-\delta,\quad \forall \delta\in (0,1),
\end{equation*}

where $\bar X=\frac 1n\sum_{i=1}^n X_i$ is the sample mean.
\end{theorem}
\begin{proof}
By direct calculation, we have
\begin{equation*}
    \sum_{i=1}^n (X_i-\bar X)^2=\sum_{i=1}^n X_i^2-2n\bar X^2+n\bar X^2=\sum_{i=1}^n X_i^2-n\bar X^2=\sum_{i=1}^n (X_i-\mu)^2-n(\bar X-\mu)^2.
\end{equation*}

Applying \Cref{thm:de la pena} to $\{X_i\}_{i=1}^n$ gives
\begin{equation*}
    \Pr\left \{\left \lvert \sum_{i=1}^n (X_i-\mu)\right \rvert\ge \sqrt{\sum_{i=1}^n (X_i-\mu)^2 \ln \frac 4\delta}\right \}\le \frac \delta 2.
\end{equation*}

Therefore, with probability $1-\frac \delta 2$, we have
\begin{equation}
    \sum_{i=1}^n (X_i-\bar X)^2\le \sum_{i=1}^n (X_i-\mu)^2\le \sum_{i=1}^n (X_i-\bar X)^2+\frac 4n \sum_{i=1}^n (X_i-\mu)^2 \ln \frac 4\delta.\label{eq:sample variance and empirical variance 1}
\end{equation}

Hence, with probability $1-\frac \delta 2$, we have
\begin{equation}
    \sum_{i=1}^n (X_i-\mu)^2\le 2\sum_{i=1}^n (X_i-\bar X)^2.\label{eq:sample variance and empirical variance 2}
\end{equation}

By the union bound, with probability $1-\delta$, we thus have
\begin{equation*}
    \left \lvert \sum_{i=1}^n (X_i-\mu)\right \rvert\le \sqrt{2\sum_{i=1}^n (X_i-\bar X)^2 \ln \frac 4\delta},
\end{equation*}

as claimed.
\end{proof}

\subsection{Sample Variance Upper Bound}
Recall that a random variable $X$ is sub-Gaussian with variance proxy $\sigma^2$ if and only if $\E[\exp(\lambda(X-\E[X]))]\le \exp(\frac 12 \lambda^2 \sigma^2)$ for all $\lambda>0$. We shall denote such a random variable by $X\sim \text{subG}(\sigma^2)$. We first state the following generalized Freedman's inequality for sub-Gaussian random variables.
\begin{proposition}[{\citet[Theorem 2.6]{fan2015exponential}}]\label{thm:generalized Freedman}
Suppose that $\{\xi_i\}_{i=1}^n$ is a sequence of zero-mean random variables, i.e., $\E[\xi_i]=0$. Suppose that $\E[\exp(\lambda \xi_i)]\le \exp(f(\lambda)V_i)$ for some deterministic function $f(\lambda)$ and some fixed $\{V_i\}_{i=1}^n$ for all $\lambda\in (0,\infty)$, then, for all $x,v>0$ and $\lambda>0$, we have
\begin{equation*}
    \Pr\left \{\exists 1\le k\le n:\sum_{i=1}^k\xi_i\ge x\wedge \sum_{i=1}^k V_i\le v^2\right \}\le \exp\left (-\lambda x+f(\lambda) v^2\right ).
\end{equation*}
\end{proposition}

To derive a bound related to $\sum(X_i-\bar X)^2$ and $\sum \sigma_i^2$, we will need to characterize the concentration of the square of a sub-Gaussian random variable, which is a ``sub-exponential'' random variable:
\begin{proposition}[{\citet[Appendix B]{honorio2014tight}}]\label{thm:sub-exponential}
For a sub-Gaussian random variable $X$ with variance proxy $\sigma^2$ and mean $\mu$, we have
\begin{equation*}
    \E\left [\exp\left (\lambda (X^2-\E[X^2])\right )\right ]\le \exp\left (16\lambda^2\sigma^4\right ),\quad \forall \lvert \lambda\rvert\le \frac{1}{4\sigma^2}.
\end{equation*}
\end{proposition}
\begin{theorem}\label{thm:variance concentration raw}
For a sequence of sub-Gaussian random variables $\{X_i\}_{i=1}^n$ such that $\E[X_i]=\mu_i$, $X_i\sim \text{subG}(\sigma_i^2)$, and $n$ is a stopping time finite a.s.,
\begin{equation*}
    \Pr\left \{\left \lvert \sum_{i=1}^n \left ((X_i-\mu_i)^2-\E[(X_i-\mu_i)^2]\right )\right \rvert> 4\sqrt 2\sum_{i=1}^n \sigma_i^2 \ln \frac 2\delta\right \}\le \delta,\quad \forall \delta\in (0,1).
\end{equation*}
\end{theorem}
\begin{proof}
We first consider a non-stopping time $n$. Apply \Cref{thm:generalized Freedman} to the sequence $\{(X_i-\mu_i)^2-\E[(X_i-\mu_i)^2]\}$ with $V_i=\sigma_i^4$, $f(\lambda)=16\lambda^2$ for $\lambda < \frac{1}{4\sigma_{\max}^2}$ and $f(\lambda)=\infty$ otherwise, where $\sigma_{\max}$ is defined as $\max\{\sigma_1,\sigma_2,\ldots,\sigma_n\}$. Then for all $x,v>0$ and $\lambda\in (0,\frac{1}{\sigma_{\max}^2})$, we have
\begin{equation*}
    \Pr\left \{\sum_{i=1}^n \left ((X_i-\mu_i)^2-\E[(X_i-\mu_i)^2]\right )>x\wedge \sqrt{\sum_{i=1}^n \sigma_i^4}\le v\right \}\le \exp \left (-\lambda x+16\lambda^2 v^2\right ).
\end{equation*}

Picking $v^2=\sum_{i=1}^n \sigma_i^4\sqrt{\ln \frac 2\delta}$ and $x=4\sqrt 2 v\ln \frac 2\delta$ gives
\begin{align*}
    &\quad \Pr\left \{\sum_{i=1}^n \left ((X_i-\mu_i)^2-\E[(X_i-\mu_i)^2]\right )>4\sqrt{2\sum_{i=1}^n \sigma_i^4}\ln \frac 2\delta\right \}\le \exp \left (-\frac{x^2}{32v^2}\right )\\
    &=\exp\left (-\frac{32v^2\ln^2\frac 2\delta}{32v^2}\right )=\frac \delta 2,
\end{align*}

where $\lambda$ is set to $\frac{x}{32v^2}=\frac{1}{4\sqrt 2 v}<\frac{1}{\sigma_{\max}^2}$ (as $v^2>\sum_{i=1}^n \sigma_i^4\ge \sigma_{\max}^4$). A union bound by applying \Cref{thm:generalized Freedman} to the sequence $\{\E[(X_i-\mu_i)^2]-(X_i-\mu_i)^2\}$ with the same parameters and noticing the fact that $\sqrt{\sum_{i=1}^n \sigma_i^4}\le \sum_{i=1}^n \sigma_i^2$ then shows that our conclusion hold for any fixed $n$. By Fatou's lemma, we conclude that it also holds for a stopping time $n$ that is finite a.s.
\end{proof}

\begin{theorem}[Variance Concentration]\label{thm:variance concentration final}
Let $\{X_i\}_{i=1}^n$ be a sequence of random variables with a common mean $\mu$ such that $X_i\sim \text{subG}(\sigma_i^2)$, $(X_i-\mu)$ is symmetric, and $n$ is a stopping time finite a.s. Then, $\forall \delta\in (0,1)$, with probability $1-\delta$, we have the following three inequalities:
\begin{align*}
    \sum_{i=1}^n (X_i-\bar X)^2&\le \sum_{i=1}^n (X_i-\mu)^2\\
    \sum_{i=1}^n (X_i-\bar X)^2&\ge \frac 12\sum_{i=1}^n (X_i-\mu)^2\\   \sum_{i=1}^n (X_i-\mu)^2&\le 8\sum_{i=1}^n \sigma_i^2 \ln \frac 2\delta.
\end{align*}

where $\bar X=\frac 1n\sum_{i=1}^n X_i$ is the sample mean.
\end{theorem}
\begin{proof}
The first two inequalities follow from \Cref{eq:sample variance and empirical variance 1,eq:sample variance and empirical variance 2}. The last one follows from \Cref{thm:variance concentration raw} together with the fact that $\E[(X_i-\mu_i)^2]\le \sigma_i^2$ by definition of sub-Gaussian random variables.
\end{proof}

\section{Ridge Linear Regression}\label{sec:ridge error bound}

\begin{lemma}\label{lem:ridge error bound}
Suppose that we are given $n$ samples $y_i=\langle  x_i,\beta^\ast\rangle+\epsilon_i$, $i=1,2,\ldots,n$, where $\beta^\ast\in \mathbb B^{d}$ and $\{x_i\}_{i=1}^n,\{\epsilon_i\}_{i=1}^n$ are stochastic processes adapted to the filtration $\{\mathcal F_i\}_{i=0}^n$ such that $\epsilon_i$ is conditionally $\bar \sigma$-Gaussian, i.e., $\epsilon_i\mid \mathcal F_{i-1}\sim \text{subG}(\bar \sigma^2)$. Define the following quantity as the estimate for $\beta^\ast$:
\begin{equation*}
    \hat \beta=\operatornamewithlimits{argmin}_{\beta}\left (\sum_{i=1}^n(y_i- x_i^\trans \beta)^2+\lambda \lVert \beta\rVert_2\right ).
\end{equation*}

Then with probability $1-\delta$, the following inequality holds:
\begin{equation*}
    \left \lvert \sum_{i=1}^n (y_i- x_i^\trans \beta^\ast)^2-\sum_{i=1}^n (y_i- x_i^\trans \hat \beta)^2\right \rvert \le 2d\bar \sigma^2\ln \frac{n}{\lambda d \delta^2}+\frac 1\lambda+\lambda.
\end{equation*}
\end{lemma}
\begin{proof}
Denote $ y=(y_1,y_2,\ldots,y_n)^\trans$, $X=(X_1^\trans,X_2^\trans,\ldots,X_n^\trans)$ and $ \epsilon=(\epsilon_1,\epsilon_2,\ldots,\epsilon_n)^\trans$. Denote $\text{Var}^\ast=\sum_{i=1}^n (y_i- x_i^\trans \beta^\ast)^2$ and $\widehat{\text{Var}}=\sum_{i=1}^n (y_i- x_i^\trans \hat \beta)^2$. We have the following representation of $\hat \beta$, which is by direct calculation (check, e.g., \citep{kirschner2018information})
\begin{equation}\label{eq:beta hat}
    \hat \beta=(X^\trans X+\lambda I)^{-1}X^\trans  y.
\end{equation}

Furthermore, by \citet[Proof of Theorem 2]{abbasi2011improved}, we have
\begin{equation}\label{eq:beta hat minus beta}
    \hat \beta-\beta^\ast=(X^\trans X+\lambda I)^{-1}X^\trans  \epsilon-\lambda (X^\trans X+\lambda I)^{-1}\beta^\ast.
\end{equation}

Therefore, we can write
\begin{align*}
\text{Var}^\ast-\widehat{\text{Var}}&=( y-X\beta^\ast)^\trans( y-X\beta^\ast)-( y-X\hat \beta)^\trans( y-X\hat \beta)\\&=(\beta^\ast)^\trans X^\trans X\beta^\ast-(\beta^\ast)^\trans X^\trans  y- y^\trans X \beta^\ast-\hat \beta^\trans X^\trans X\hat \beta+\hat \beta^\trans X^\trans  y+ y^\trans X \hat \beta\\&=\hat \beta^\trans(X^\trans  y-X^\trans X\hat \beta)-(\beta^\ast)^\trans X^\trans( y-X\beta^\ast)+ y^\trans X (\hat \beta-\beta^\ast).
\end{align*}

By \Cref{eq:beta hat}, we have $X^\trans  y=(X^\trans X+\lambda I)\hat \beta$. So the first term is just $\lambda \hat \beta^\trans \hat \beta$. As $ y=X\beta^\ast+ \epsilon$, the second term is just $-(\beta^\ast)^\trans X^\trans  \epsilon$. By \Cref{eq:beta hat minus beta}, the last term becomes
\begin{equation*}
     y^\trans X(X^\trans X+\lambda I)^{-1}X^\trans  \epsilon-\lambda  y^\trans X(X^\trans X+\lambda I)^{-1}\beta^\ast.
\end{equation*}

For the sake of simplicity, we define $\langle  a, b\rangle_M= a^\trans (X^\trans X+\lambda I)^{-1} b$ (note that $\langle  a, b\rangle_M=\langle  b, a\rangle_M$ as $X^\trans X+\lambda I$ is symmetric) and denote the induced norm by $\lVert \cdot\rVert_M$. Therefore, we have
\begin{equation*}
    \text{Var}^\ast-\widehat{\text{Var}}=\lambda \hat \beta^\trans \hat \beta-(\beta^\ast)^\trans  X^\trans  \epsilon+\langle X^\trans  y,X^\trans  \epsilon\rangle_M-\lambda \langle X^\trans y,\beta^\ast\rangle_M.
\end{equation*}

Again by \Cref{eq:beta hat}, we have $\langle X^\trans  y,X^\trans  \epsilon\rangle_M=\hat \beta^\trans X^\trans  \epsilon$. Therefore, the second and third term together give
\begin{align*}
    &\quad -(\beta^\ast)^\trans  X^\trans  \epsilon+\langle X^\trans  y,X^\trans  \epsilon\rangle_M=(\hat \beta-\beta^\ast)^\trans  X^\trans  \epsilon\\&=\left ((X^\trans X+\lambda I)^{-1}X^\trans  \epsilon-\lambda (X^\trans X+\lambda I)^{-1}\beta^\ast\right )^\trans X^\trans  \epsilon\\&=\lVert X^\trans  \epsilon\rVert_M^2-\lambda \langle \beta^\ast,X^\trans  \epsilon\rangle=\lVert X^\trans  \epsilon\rVert_M^2-\lambda (\beta^\ast)^\trans (\hat \beta-\beta^\ast)+\lambda^2 \langle \beta^\ast,\beta^\ast\rangle_M,
\end{align*}

where the last step is yielded from using \Cref{eq:beta hat minus beta} reversely. Then note that $\langle X^\trans y,\beta^\ast\rangle=\hat \beta^\ast \beta^\ast$, so taking expectation on both sides gives
\begin{align*}
    \left \lvert\text{Var}^\ast-\widehat{\text{Var}}\right \rvert&=\lambda \hat \beta^\trans \hat \beta+\lVert X^\trans  \epsilon\rVert_M^2-\lambda (\beta^\ast)^\trans (\hat \beta-\beta^\ast)+\lambda^2 \langle \beta^\ast,\beta^\ast\rangle_M-\lambda \hat \beta^\trans \beta^\ast\\
    &=\lVert X^T\epsilon\rVert_M^2+\lambda \lVert \hat \beta-\beta^\ast\rVert_2^2+\lambda^2 \lVert \beta^\ast\rVert_M^2.
\end{align*}

By Cauchy-Schwartz inequality, we have
\begin{equation*}
    \lVert \hat \beta-\beta^\ast\rVert_2^2\le \lVert X^\trans  \epsilon\rVert_M^2\lVert (X^\trans X+\lambda I)^{-1/2}\rVert^2+\lVert \beta^\ast\rVert^2_M\lVert (X^\trans X+\lambda I)^{-1/2}\rVert^2,
\end{equation*}

where the matrix norms can further be bounded by $1/\lambda_{\min}(X^\trans X+\lambda I)\le \frac 1\lambda$. Similarly, we can conclude that $\lVert \beta^\ast\rVert_M^2\le 1/\lambda_{\min}(X^\trans X+\lambda I) \lVert \beta^\ast\rVert_2^2\le 1/\lambda$. Consequently, we have
\begin{equation*}
    \left \lvert\text{Var}^\ast-\widehat{\text{Var}}\right \rvert=2\lVert X^T \epsilon\rVert_M^2+\frac{1}{\lambda}+\lambda.
\end{equation*}

As proved by \citet[Theorem 1]{abbasi2011improved}, with probability $1-\delta$, we have
\begin{equation*}
    \lVert X^\trans  \epsilon\rVert_M^2\le 2\bar \sigma^2\ln \left (\frac 1\delta \frac{\sqrt{\det(X^\trans X+\lambda I)}}{\sqrt{\det(\lambda I)}}\right )=\bar \sigma^2\ln \left (\frac 1{\delta^2} \frac{(\lambda+n/s)^s}{\lambda}\right )\le d\bar \sigma^2\ln \left (\frac{1}{\delta^2}\left (1+\frac{n}{\lambda d}\right )\right ),
\end{equation*}

where $\bar \sigma^2$ is the (maximum) variance proxy and the second last step is due to the Determinant-Trace Inequality \citep[Lemma 10]{abbasi2011improved}. Hence, with probability $1-\delta$, we have
\begin{equation*}
   \left \lvert\text{Var}^\ast-\widehat{\text{Var}}\right \rvert\le 2d\bar \sigma^2\ln \frac{n}{\lambda d \delta^2}+\frac 1\lambda+\lambda,
\end{equation*}

as claimed.
\end{proof}


\section{Omitted Proof in \Cref{sec:analysis} (Analysis of Framework)}\label{sec:appendix analysis}

\subsection{Proof of Main Theorem}
In this section, we prove \Cref{thm:total regret of part b}, which is restated as follows for the ease of reading:
\begin{theorem}[Restatement of \Cref{thm:total regret of part b}]\label{thm:appendix total regret of part b}
Suppose that for any execution of $\mathcal F$ that last for $n$ steps, $\bar{\mathcal R_n^{\mathcal F}}\ge \mathcal R_n^{\mathcal F}$ holds with probability $1-\delta$, i.e., $\bar{\mathcal R_n^{\mathcal F}}$ is pessimistic. Then the total regret incurred by the second phase satisfies
\begin{align*}
    \mathcal R_T^b&= \Otil\left (s\sqrt d\sqrt{\sum_{t=1}^T \sigma_t^2}\log \frac 1\delta+s\log \frac 1\delta\right )\quad \text{with probability }1-\delta.
\end{align*}
\end{theorem}

\begin{proof}
As mentioned in the main body, we define the following good events for each $\Delta=2^{-2},2^{-2},\ldots,\Delta_f$ where $\Delta_f$ is the final value of $\Delta$:
\begin{align}
    \mathcal G_\Delta&:\sum_{i\in S}\theta_i^\ast (\theta_i^\ast-\hat \theta_i)<\Delta^2\text{ after }\mathcal T_\Delta^a,\nonumber\\
    \mathcal H_\Delta&:\left (\lvert \theta_i^\ast\rvert>3\Delta\to i\in S\right )\wedge \left (i\in S\to \lvert \theta_i^\ast\rvert>\frac 12\Delta\right )\text{ after }\mathcal T_\Delta^b.\label{eq:good events}
\end{align}

It is by definition to see that $\mathcal H_{\frac 12}$ indeed holds, as all $\lvert \theta_i^\ast\rvert < \frac 32$ and $S$ is initially empty. We can then use induction to prove that all good events hold with high probability.

Here, we list several technical lemmas we informally referred to in the main body, whose proofs are left to subsequent sections. The first one is about the regret-to-sample-complexity conversion.
\begin{lemma}[Regret-to-sample-complexity Conversion]\label{lem:appendix regret-to-sc}
If for any execution of $\mathcal F$ that lasts for $n$ steps, we have $\bar{\mathcal R_n^{\mathcal F}}\ge \mathcal R_n^{\mathcal F}$ with probability $1-\delta$, then we have $\Pr\{\mathcal G_\Delta\}\ge 1-\delta$.
\end{lemma}

The second term bounds the `bias' term, i.e., the second term of \Cref{eq:random projection Z_n^k} for random projection.
\begin{lemma}[Bias Term of the Random Projection]\label{lem:appendix bias term of random projection}
Conditioning on $\mathcal G_\Delta$ and $\mathcal H_{2\Delta}$, we have $\lvert \sum_{i\in S}(\hat \theta_i^2-(\theta_i^\ast)^2)\rvert<3\Delta^2$, which further gives $\lvert \sum_{i\in S}\hat \theta_i(\theta_i^\ast-\hat \theta_i)\rvert<4\Delta^2$.
\end{lemma}

Furthermore, we can bound the estimation error of the random projection process.
\begin{lemma}[Concentration of the Random Projection]\label{lem:appendix concentration of random projection}
For any given $i\in [K]$, we will have
\begin{equation*}
    \Pr\left \{\lvert \bar r_i-\theta_i^\ast\rvert>3\Delta^2+\frac \Delta 4\middle \vert \mathcal G_\Delta,\mathcal H_{2\Delta}\right \}\le \delta.
\end{equation*}
\end{lemma}

As long as the estimation errors are small, we can ensure, with high probability that the good event for $\Delta$, namely $\mathcal H_\Delta$ will also hold.
\begin{lemma}[Identification of Non-zero Coordinates]\label{lem:appendix identification of coordinates}
$\Pr\{\mathcal H_\Delta\mid \mathcal G_\Delta,H_{2\Delta}\}\ge 1-d\delta$.
\end{lemma}

Therefore, by combining all lemmas above, we can ensure that all good events, namely $\{\mathcal G_\Delta\wedge \mathcal H_\Delta\}_{\Delta=2^{-2},\ldots,\Delta_f}$, hold simultaneously with probability $1-dT\delta$.

At last, we bound the total regret incurred in Phase B for $\Delta$.
\begin{lemma}[Single-Phase Regret Bound]
\label{lem:appendix single step regret bound}
Conditioning on $\mathcal G_\Delta$ and $\mathcal H_{2\Delta}$, we have
\begin{equation*}
    \sum_{t\in \mathcal T_\Delta^b}\langle  \theta^\ast- x_t, \theta^\ast\rangle\le 36sn_\Delta^b \Delta^2+6s\Delta\frac{R}{\sqrt K}\sqrt{2n_\Delta^b \ln \frac 1\delta},\quad \text{with probability $1-s\delta$}.
\end{equation*}
\end{lemma}

Then we follow the analysis sketched in the main body. We assume, without loss of generality, that $s<d$. Then, conditioning on all $\mathcal G_\Delta$ and $\mathcal H_\Delta$, some coordinate $i^\ast$ must never be included into $S$ as $\mathcal H_\Delta$ holds for all $\Delta$.

Therefore, by property of the sample variances (\Cref{thm:variance concentration final}), for such $i^\ast$ and for each phase, with probability $1-\delta$, we have $\sum_{n=1}^N (r_{n,i^\ast}-\bar r_{i^\ast})^2\le \sum_{n=1}^N (r_{n,i^\ast}-\mathbb E[r_{n,i^\ast}])^2$. Together with \Cref{eq:designing N in Option II},
\begin{equation}\label{eq:upper bound of n_Delta^b}
    2\sqrt{2\sum_{t\in \tilde{\mathcal T_\Delta^b}} (r_{t,i^\ast}-\mathbb E[r_{t,i^\ast}])^2\ln \frac 4\delta}>(n_\Delta^b-1)\frac \Delta 4.
\end{equation}

By using \Cref{lem:appendix single step regret bound}, the total regret from Phase B is bounded by
\begin{equation*}
    \mathcal R_T^b\le \sum_{\Delta=2^0,2^{-2},\ldots,\Delta_f} \left (36sn_\Delta^b\Delta^2+6s\Delta\frac{R}{\sqrt K}\sqrt{2n_\Delta^b\ln \frac 1\delta}\right )
\end{equation*}

By writing $n_\Delta^b$ as $(n_\Delta^b-1)+1$, for any given $\Delta$, the total regret of Phase b is bounded by
\begin{align}
    \mathcal R_T^b&\le \sum_\Delta\sum_{t\in \mathcal T_\Delta^b}\langle  \theta^\ast- x_t, \theta^\ast\rangle\nonumber\\
    &\le \sum_\Delta36s\Delta^2 \cdot \left (\frac 4\Delta 2\sqrt{2\sum_{t\in \tilde{\mathcal T_\Delta^b}} (r_{t,i^\ast}-\mathbb E[r_{t,i^\ast}])^2\ln \frac 4\delta}+1\right )+\nonumber\\
    &\quad \sum_\Delta6s\Delta \frac{R}{\sqrt K}\sqrt{2\ln \frac 1\delta}\sqrt{4\Delta\cdot  2\sqrt{2\sum_{t\in \tilde{\mathcal T_\Delta^b}} (r_{t,i^\ast}-\mathbb E[r_{t,i^\ast}])^2\ln \frac 4\delta}+1}\nonumber\\
    &\le \underbrace{\sum_\Delta\left (288\sqrt 2 s\sqrt{\Delta^2\sum_{t\in \mathcal T_\Delta^b}(r_{t,i^\ast}-\mathbb E[r_{t,i^\ast}])^2\ln \frac 4\delta}+36s\Delta^2\right )}_{\text{Part (a)}}+\label{eq:sum of phase B regret term 1}\\
    &\quad \underbrace{\sum_\Delta\left (24s\sqrt{\ln \frac 1\delta}\Delta\sqrt[4]{\Delta^2\sum_{t\in \mathcal T_\Delta^b}(r_{t,i^\ast}-\E[r_{t,i^\ast}])^2\ln \frac 4\delta}+12s\Delta \sqrt{\ln \frac 1\delta}\right )}_{\text{Part (b)}}.\label{eq:sum of phase B regret term 2}
\end{align}


As mentioned in the main text, we make use of the following lower bound of $n_\Delta^b$ which again follows from \Cref{eq:designing N in Option II} and holds for all $\Delta$'s except $\Delta_f$:
\begin{equation*}
    n_\Delta^b \ge \frac 4 \Delta\cdot  2\sqrt{2\sum_{t\in \mathcal T_\Delta^b}(r_{t,i^\ast}-\bar r_{i^\ast})^2\ln \frac 4\delta}\ge \frac{8}{\Delta}\sqrt{2\sum_{t\in \mathcal T_\Delta^b}(r_{t,i^\ast}-\E[r_{t,i^\ast}])^2\ln \frac 4\delta},
\end{equation*}

where the last step is due to \Cref{thm:variance concentration final}. For simplicity, define $S_\Delta=\Delta^2\sum_{t\in \mathcal T_\Delta^b} (r_{n,i}-\E[r_{n,i}])^2$ and therefore $n_\Delta^b\ge \frac{C}{\Delta^2}\sqrt{S_\Delta}$ where $C=8\sqrt{2\ln \frac 4\delta}$, a constant if we regard $\delta$ as a constant. Therefore,
\begin{equation}\label{eq:upper bound of sum n}
    \sum_{\Delta\ne \Delta_f}\frac{C}{\Delta^2}\sqrt{S_\Delta}\le \sum_{\Delta\ne \Delta_f}n_\Delta^b\le T
\end{equation}

and our goal is to upper bound $\sum_{\Delta}\sqrt{S_\Delta}$. Define a threshold $X=T/\left (C\sqrt{\sum_{\Delta\ne \Delta_f}S_\Delta}\right )$ and denote $\Delta_X=2^{-\lceil \log_4 X\rceil}$ so that $\Delta_X^{2}\le \frac 1X$. We will have
\begin{align}
    \sum_{\Delta\ne \Delta_f}\sqrt{S_\Delta}&=\sum_{\Delta=2^{-2},\ldots,\Delta_X}\sqrt{S_\Delta} + \sum_{\Delta=\Delta_X/2,\ldots,2\Delta_f}\sqrt{S_\Delta}\nonumber\\
    &\overset{(a)}{\le} \sqrt{\log_4X}\sqrt{\sum_{\Delta=2^{-2},\ldots,\Delta_X}S_\Delta}+\sum_{\Delta_X>\Delta>\Delta_f}\frac{\Delta_X^2}{\Delta_X^{2}}\sqrt{S_\Delta}\nonumber\\
    &\overset{(b)}{\le} \sqrt{\log_4X} \sqrt{\sum_{\frac 14\ge \Delta\ge \Delta_X}S_\Delta}+\frac 1X\sum_{\Delta_X>\Delta>\Delta_f}\frac{\sqrt{S_\Delta}}{\Delta_X^2}\nonumber\\
    &\overset{(c)}{\le} \sqrt{\log_4X} \sqrt{\sum_{\frac 14\ge \Delta\ge \Delta_X}S_\Delta}+\frac 1X\frac TC=(\sqrt{\log_4X}+1)\sqrt{\sum_{\Delta\ne \Delta_f}S_\Delta},\label{eq:summation technique}
\end{align}

where (a) applied Cauchy-Schwartz to the first summation, (b) applied the fact that $\Delta_X^2\le \frac 1X$ and (c) used \Cref{eq:upper bound of sum n}. So Part (a) of the regret, namely \Cref{eq:sum of phase B regret term 1}, can be bounded by
\begin{align}
    &\quad \sum_\Delta\left (288\sqrt 2 s\sqrt{\Delta^2\sum_{t\in \mathcal T_\Delta^b}(r_{t,i}-\mathbb E[r_{t,i}])^2\ln \frac 4\delta}+36s\Delta^2\right )\nonumber\\
    &\overset{(a)}{\le} 288\sqrt 2s \sum_\Delta \sqrt{S_\Delta \ln \frac 4\delta}+72s\overset{(b)}{\le}288\sqrt 2s\sqrt{\ln \frac 4\delta}\left (\sqrt{\log_4(4T)\sum_{\Delta\ne \Delta_f} S_\Delta}+\sqrt{S_{\Delta_f}}\right )+72s\nonumber\\
    &\overset{(c)}{\le}576 s \sqrt{\sum_{\Delta=2^{-2},\ldots,\Delta_f}\Delta^2 \sum_{t\in \mathcal T_\Delta^b}(r_{t,i^\ast}-\E[r_{t,i^\ast}])^2\ln \frac 4\delta\log_4(4T)}+72s\nonumber\\
    &\overset{(d)}{\le}576 s \sqrt{\sum_{\Delta=2^{-2},\ldots,\Delta_f}8\sum_{t\in \mathcal T_\Delta^b}\left (\Delta^2+d\sigma_t^2\right )\ln \frac 4\delta\log_4(4T)}+72s\label{eq:Phase b eq d}\\
    &\le 1152\sqrt 2 s\sqrt d \sqrt{\sum_{t=1}^T \sigma_t^2 \ln \frac 4\delta\log_4(4T)}+1152\sqrt 2 s \sqrt{\sum_{\Delta} n_\Delta^b \Delta^2\ln \frac 1\delta\log_4(4T)} +72s\nonumber.
\end{align}

where (a) used the fact that $\sum \Delta^2\le 2$ as $\Delta=2^{-i}$, (b) used \Cref{eq:summation technique}, (c) again used Cauchy-Schwartz inequality and (d) used \Cref{thm:variance concentration final}(3), the variance concentration result, \Cref{eq:variance of feedback}, the magnitude of the variance, and the facts that $R\ge \Delta$, $K\le d$. Therefore, we can conclude that
\begin{equation*}
    \O(s)\cdot \sum_\Delta n_\Delta^b \Delta^2\le \Otil\left (s\sqrt d\sqrt{\sum_{t=1}^T\sigma_t^2 \log \frac 4\delta}+s\right )+\Otil\left (s\sqrt{n_\Delta^b \Delta^2 \ln \frac 1\delta}\right ).
\end{equation*}

Notice that the left-handed-side and the right-handed-side has a common term (the RHS one is inside the square-root sign). Hence, by the self-bounding property \Cref{lem:self-bounding property from x<a+bsqrt(x)}, we can conclude that (note that we divided $s$ on both sides)
\begin{equation}
    \sum_\Delta n_\Delta^b \Delta^2\le \Otil\left (\sqrt d\sqrt{\sum_{t=1}^T\sigma_t^2 \log \frac 4\delta}+1\right )+\Otil\left (\ln \frac 1\delta\right ),\label{eq:Phase b after self-bounding}
\end{equation}

which means that Part (a) (\Cref{eq:sum of phase B regret term 1}), or equivalently $\O(s\sum_\Delta n_\Delta^b \Delta^2)$, is bounded by
\begin{equation*}
    \Otil\left (s\sqrt d\sqrt{\sum_{t=1}^T \sigma_t^2\log \frac 4\delta}+s\log \frac 4\delta\right ).
\end{equation*}

Now consider Part (b) of the regret, namely \Cref{eq:sum of phase B regret term 2}. The second term in each summand will sum up to $\O(s\sqrt{\log \frac 1\delta})$. For the first term, using the notation $S_\Delta$, we want to bound
\begin{align*}
    &\quad 24s \left (\ln \frac 1\delta\right )^{\nicefrac 34} \sum_\Delta \Delta \sqrt[4]{S_\Delta}\le 24s \ln \frac 1\delta \sum_\Delta \Delta \sqrt[4]{\sum_{\Delta}S_\Delta}\le 48s \ln \frac 1\delta\sqrt[4]{\sum_{\Delta}S_\Delta}.
\end{align*}

Conditioning on the same events as in \Cref{eq:Phase b eq d}, we will have
\begin{equation*}
    \sum_\Delta S_\Delta\le \O\left (\sum_\Delta n_\Delta^b \Delta^2+\sum_{t=1}^T d\sigma_t^2\right )\le \Otil\left (\sum_{t=1}^T d\sigma_t^2+\sqrt{\sum_{t=1}^T d\sigma_t^2\log \frac 1\delta}+\log \frac 1\delta\right )\le \Otil\left (\sum_{t=1}^T d\sigma_t^2\log \frac 1\delta\right ),
\end{equation*}

where the second step comes from \Cref{eq:Phase b after self-bounding}. Therefore,

\begin{align*}
    24s \left (\ln \frac 1\delta\right )^{\nicefrac 34} \sum_\Delta \Delta \sqrt[4]{S_\Delta}\le \Otil\left (s\sqrt{\sqrt d\sqrt{\sum_{t=1}^T \sigma_t^2\log \frac 1\delta}}\left (\log \frac 1\delta\right )^{\nicefrac 34}\right )=\Otil\left (s\sqrt[4]{d\sum_{t=1}^T \sigma_t^2}\log \frac 1\delta\right ),
\end{align*}

which gives (by combining the bound of \Cref{eq:sum of phase B regret term 1} and \Cref{eq:sum of phase B regret term 2} together):
\begin{equation*}
    \mathcal R_T^b\le \Otil\left (s\sqrt d\sqrt{\sum_{t=1}^T \sigma_t^2\log \frac 4\delta}+s\sqrt[4]{d\sum_{t=1}^T \sigma_t^2}\log \frac 1\delta+s\log \frac 4\delta\right ).
\end{equation*}

Further notice that, if $d\sum_{t=1}^T \sigma_t^2\ge 1$, then the square-root is larger than the 4th-root. When $d\sum_{t=1}^T \sigma_t^2\le 1$, then either root is bounded by $s$. Hence, in either case, the 4th-root will be hidden by other factors. Henceforth, we indeed have the following conclusion with probability $1-(s+3T)\delta$:
\begin{equation*}
    \mathcal R_T^b\le \Otil\left (s\sqrt{d\sum_{t=1}^T \sigma_t^2}\log \frac 1\delta+s\log \frac 1\delta\right ).
\end{equation*}

By setting the actual $\delta$ as $(s+3T)\delta$, we will still have the same regret bound as $\O(\log \frac c\delta)=\O(\log \frac 1\delta+\log c)=\Otil(\log \frac 1\delta)$ as all logarithmic factors will be hidden by $\Otil$.
\end{proof}

\subsection{Regret-to-Sample-Complexity Conversion}
\begin{proof}[Proof of \Cref{lem:appendix regret-to-sc}]
By Fatou's lemma and the fact that $n_\Delta^b$ is finite a.s. (as it is truncated according to $T$), the probability that $\bar{\mathcal R_n^{\mathcal F}}$ is a pessimistic estimation for all $n=1,2,\ldots,n_\Delta^b$ is bounded by $1-\delta$. Conditioning on this, by definition, we will have
\begin{equation*}
    \sum_{n=1}^{n_\Delta^b} \langle  \theta_i^\ast, \theta_i^\ast- x_n\rangle\le \mathcal R_{n_\Delta^b}^{\mathcal F}\le \bar{\mathcal R_{n_\Delta^b}^{\mathcal F}},
\end{equation*}

which means our stopping criterion (\Cref{line:phase a terminiate}) will ensure
\begin{equation*}
    \sum_{n=1}^{n_\Delta^b} \langle  \theta_i^\ast, \theta_i^\ast-\hat \theta\rangle\le \frac{\bar{\mathcal R_{n_\Delta^b}^{\mathcal F}}}{n_\Delta^b}<\Delta^2
\end{equation*}

and thus $\mathcal G_\Delta$ is ensured.
\end{proof}

\subsection{Random Projection}
\begin{proof}[Proof of \Cref{lem:appendix bias term of random projection}]
By $\mathcal H_{2\Delta}$, we have $\lvert \theta_i^\ast\rvert>\Delta$, $\forall i\in S$, which gives $\lvert \theta_i^\ast-\hat \theta_i\rvert<\Delta$ by $\mathcal G_\Delta$. So we can bound $\lvert \sum_{i\in S}(\hat \theta_i^2 -(\theta_i^\ast)^2)\rvert$ by $2\lvert \sum_{i\in S}\theta_i^\ast(\theta_i^\ast-\hat \theta_i)\rvert+\lvert \sum_{i\in S}(\theta_i^\ast-\hat \theta_i)^2\rvert<2\Delta^2+\Delta^2=3\Delta^2$. The second claim is then straightforward as $\hat \theta_i(\theta_i^\ast-\hat \theta_i)=-\theta_i^\ast(\theta_i^\ast-\hat \theta_i)-\left ((\hat \theta_i^2-(\theta_i^\ast)^2)\right )$.
\end{proof}

\begin{proof}[Proof of \Cref{lem:appendix concentration of random projection}]
In the random projection procedure, let $r_i^{(n)}$ be the random variable defined as (which is the $i$-th coordinate of $r_n$ in the algorithm)
\begin{align}
    r_i^{(n)}&=\frac{K}{R^2}y_i^{(n)}\left (\langle  x_n, \theta^\ast\rangle+\eta_n-\sum_{j\in S}\hat \theta_j^2\right )\nonumber\\
    &=\frac{K}{R^2}\left (y_i^{(n)}\right )^2\theta_i^\ast+\sum_{j\in S}\hat \theta_j(\theta_j^\ast-\hat \theta_j)+\sum_{j\notin S,j\ne i}\left (\frac{K}{R^2}y_i^{(n)}y_j^{(n)}\right )\theta_j^\ast+\left (\frac{K}{R^2}y_i^{(n)}\right )\eta_n.\label{eq:Z_i^n}
\end{align}

Then $\bar r_i$ is just the sample mean estimator of $\{r_i^{(n)}\}_{n=1}^{n_\Delta^b}$.

Firstly, observe that $y_i^2$ is always $\frac{R^2}{K}$, so the first term of \Cref{eq:Z_i^n} is just $\theta_i^\ast$, which is exactly the magnitude we want to estimate. Moreover, by \Cref{lem:appendix bias term of random projection}, the second term is a small (deterministic) bias added to $\theta_i^\ast$ and is bounded by $3\Delta^2$. For the third term, as each $y_j$ is independent, $\frac{K}{R^2}y_iy_j$ is an i.i.d. Rademacher random variable, denoted by $z_j^{(n)}$. Hence, the last two terms sum to
\begin{equation*}
    \sum_{j\notin S,j\ne i}z_j^{(n)}\theta_j^\ast+\left (\frac{K}{R^2}y_i^{(n)}\right )\eta_n.
\end{equation*}

By definition of Rademacher random variables, they are all zero-mean and symmetric. Henceforth they can be viewed as noises with variances at most 
\begin{equation}
    \E\left [\left (\sum_{j\not\in S,j\ne i}z_j^{(n)}\theta_j^\ast+\frac{K}{R^2}y_i^{(n)}\eta_n\right )^2\right ]\overset{(a)}{=}\E\left [\sum_{j\not \in S}\left (z_j^{(n)}\theta_j^\ast\right )^2+\left (\frac{K}{R^2}y_i^{(n)}\right )^2\eta_n^2\right ]\le 1+\frac{K}{R^2}\sigma_n^2,\label{eq:variance of feedback}
\end{equation}

where (a) is due to the mutual independence between $z_j^{(n)}$ and $y_i^{(n)}$.

Moreover, as each $y_i^{(n)}$ is also redrawn for every $n$ and $i$, we know that all $\{z_j^{(n)}\}_{j\not \in S,1\le n\le n_\Delta^b}$ and $\{y_i^{(n)}\}_{1\le n\le n_\Delta^b}$ are all mutually independent. Because the first two terms \Cref{eq:Z_i^n} is not random, all $\{r_i^{(n)}\}_{1\le n\le n_\Delta^b}$ are independent, symmetric and sub-Gaussian (as $\theta_j^\ast$ is bounded and $\eta_n$ is Gaussian) random variables.
Therefore, as $n_\Delta^b$ is indeed a stopping time finite a.s., we shall apply Empirical Bernstein Inequality (\Cref{thm:Bernstein with common mean}) to $\{r_i^{(n)}\}_{n}$ where $\text{Var}(r_i^{(n)})$ is characterized by \Cref{eq:variance of feedback}, giving
\begin{align*}
    &\Pr\left \{ \left\lvert \sum_{n=1}^{n_\Delta^b}  \left (\sum_{j\notin S,j\ne i}z_j^{(n)}\theta_j^\ast+\left (\frac{K}{R^2}y_i^{(n)}\right )\eta_n\right )\right \rvert\ge 2\sqrt{2\sum_{n=1}^{n_\Delta^b} (r_{n,i}-\bar r_i)^2\ln \frac 4\delta}\right \}\le \delta.
\end{align*}

In other words, our choice of $n_\Delta^b$ in \Cref{eq:designing N in Option II} will ensure the average noise is bounded by $\frac{\Delta}{2}$. By \Cref{lem:appendix bias term of random projection}, we conclude that $\Pr\{\lvert \bar r_i-\theta_i^\ast\rvert>3\Delta^2+\frac \Delta 4\mid \mathcal G_\Delta,\mathcal H_{2\Delta}\}\le \delta$.
\end{proof}

\begin{proof}[Proof of \Cref{lem:appendix identification of coordinates}]
If we skipped due to $\sum_{i\in S}\hat \theta_i^2>1-\Delta^2$, which is \Cref{line:skipping} of \Cref{alg:framework}, then by \Cref{lem:appendix bias term of random projection}, we will have
\begin{equation*}
    \sum_{i\in S}(\theta_i^\ast)^2>1-5\Delta^2\quad \text{conditioning on }\mathcal G_\Delta\wedge \mathcal H_{2\Delta},
\end{equation*}

and thus all remaining coordinates are smaller than $3\Delta$. Moreover, by $\mathcal H_{2\Delta}$, all discovered coordinates are with magnitude at least $\frac{2\Delta}{2}>\frac \Delta 2$. Hence $\mathcal H_\Delta$ automatically holds in this case.

Otherwise, suppose that the conclusion of \Cref{lem:appendix concentration of random projection} holds for all $i\in [K]$ (which happens with probability $1-K\delta$ conditioning on $\mathcal G_\Delta$ and $\mathcal H_{2\Delta}$). As we only pick those coordinates with $\lvert \bar r_i\rvert>\Delta$, all coordinates with magnitude at last $\Delta+\frac \Delta 4+3\Delta^2< 3\Delta$ will be picked as $\Delta\le \frac 14$, so the first condition of $\mathcal H_\Delta$ indeed holds. Moreover, all picked coordinates will have magnitudes at least $\Delta- (\frac \Delta 4+3\Delta^2)$. But when $\Delta=\frac 14$, there will be no coordinates in $S$ as it is initially empty. And after that, we will have $\Delta^2\le \frac \Delta 8$. Hence, all coordinates with magnitude $\frac \Delta 2$ will surely be identified, which means the second condition of $\mathcal H_\Delta$ also holds.

We then have $\Pr\{\mathcal H_\Delta\mid \mathcal G_\Delta,\mathcal H_{2\Delta}\}\ge 1-K\delta$ by putting these two cases together.
\end{proof}

\subsection{Single-Phase Regret}
\begin{proof}[Proof of \Cref{lem:appendix single step regret bound}]
Conditioning on $\mathcal G_\Delta$, as we are playing $x_{t,i}=\hat \theta_i$ for all $i\in S$ and $t\in \mathcal T_\Delta^b$, we will have
\begin{equation*}
    \sum_{t\in \mathcal T_\Delta^b}\sum_{i\in S}(\theta_i^\ast-x_{t,i})\theta_i^\ast=\sum_{t\in \mathcal T_\Delta^b}\sum_{i\in S}(\theta_i^\ast-\hat \theta_i)\theta_i^\ast<n_\Delta^b\cdot \Delta^2.
\end{equation*}

Now consider a single coordinate not in $S$. For each $t\in \mathcal T_\Delta^b$, we will equiprobably play $\pm \frac{R}{\sqrt K}$ for this coordinate. Hence, the total regret will become
\begin{equation*}
    \sum_{t\in \mathcal T_\Delta^b}\left (\theta_i^\ast \pm \frac{R}{\sqrt K}\right )\theta_i^\ast=n_\Delta^b\cdot \left (\theta_i^\ast\right )^2+\theta_i^\ast \sum_{t\in \mathcal T_\Delta^b}\left (\pm \frac{R}{\sqrt K}\right ).
\end{equation*}

By $\mathcal H_{2\Delta}$, the first term is bounded by $36n_\Delta^b \Delta^2$. By Chernoff bound, the absolute value of the summation in the second term will be bounded by $\frac{R}{\sqrt K}\sqrt{2n_\Delta^b \ln \frac 1\delta}$ with probability $1-\delta$. As there are at most $s$ non-zero coordinates by sparsity, from a union bound, we can conclude that
\begin{equation*}
    \sum_{t\in \mathcal T_\Delta^b} \langle  \theta^\ast- x_t, \theta^\ast\rangle\le 36sn_\Delta^b \Delta^2+6s\Delta\frac{R}{\sqrt K}\sqrt{2n_\Delta^b \ln \frac 1\delta}
\end{equation*}

with probability $1-s\delta$.
\end{proof}

\section{Omitted Proof in \Cref{sec:zhang et al} (Analysis of \texttt{VOFUL2})}\label{sec:appendix zhang et al}

\subsection{Proof of Main Theorem}\label{sec:appendix zhang et al unknown}
In this section, we prove \Cref{thm:regret with VOFUL2 and unknown variances}, which is restated as \Cref{thm:appendix VOFUL2 unknown}. We first assume that \Cref{prop:VOFUL2 regret} is indeed correct, whose discussion is left to \Cref{sec:VOFUL2 regret upper bound}.

\begin{theorem}[Regret of \Cref{alg:framework} with \texttt{VOFUL2} in Unkown-Variance Case]\label{thm:appendix VOFUL2 unknown}
Consider \Cref{alg:framework} with $\mathcal F$ as \texttt{VOFUL2} \citep{kim2021improved} and $\bar{\mathcal R_n^{\mathcal F}}$ as
\begin{equation}\label{eq:regret estimation with unknown variance}
    \bar{\mathcal R_n^{\mathcal F}}=C\left (s^{1.5}\sqrt{\sum_{k=1}^n (r_k-\langle  x_k, \hat \beta\rangle)^2\ln \frac 1\delta}+s^2\sqrt{2\ln \frac{n}{s\delta^2}\ln \frac 1\delta}+s^{1.5}\sqrt{2\ln \frac 1\delta}+s^2\ln \frac 1\delta\right ),
\end{equation}

where $C=\Otil(1)$ is the constant hidden in the $\Otil$ notation of \Cref{prop:VOFUL2 regret}, $ x_1, x_2,\ldots, x_n$ are the actions made by the agent, $r_1,r_2,\ldots,r_n$ are the corresponding (noisy) feedback, and $\hat \beta$ is defined as
\begin{equation*}
    \hat \beta=\operatornamewithlimits{argmin}_{ \beta\in \mathbb R^s}\left (\sum_{k=1}^n(r_k-\langle x_k, \beta\rangle)^2+\lVert  \beta\rVert_2\right ).
\end{equation*}

The algorithm ensures the following regret bound with probability $1-\delta$:
\begin{equation*}
    \mathcal R_T= \Otil\left ((s^{2.5}+s\sqrt d)\sqrt{\sum_{t=1}^T \sigma_t^2 \log \frac 1\delta}+s^3\log \frac 1\delta\right ).
\end{equation*}
\end{theorem}

\begin{proof}
By \Cref{prop:VOFUL2 regret}, the total regret incurred by algorithm $\mathcal F$ satisfies
\begin{equation*}
    \mathcal R_n^{\mathcal F}\le C\left (s^{1.5}\sqrt{\sum_{k=1}^n \eta_k^2 \ln \frac 1\delta}+s^2\ln \frac 1\delta\right ),
\end{equation*}

where $C=\Otil(1)$ is a constant (with some logarithmic factors) and $\eta_k\sim \mathcal N(0,\sigma_k^2)$ is the noise for the $k$-th round executing $\mathcal F$. We now consider our regret estimator $\bar{\mathcal R_n^{\mathcal F}}$. We show that, with high probability, it is a pessimistic estimation.

\begin{lemma}\label{lem:Zhang unknown is overestimation}
For any given $\Delta$, with probability $1-\delta$, $\bar{\mathcal R_n^{\mathcal F}}\ge \mathcal R_n^{\mathcal F}$.
\end{lemma}

Therefore, for the ``explore'' phase, we can make use of \Cref{thm:appendix total regret of part b} which only requires $\bar{\mathcal R_n^{\mathcal F}}$ to be an over-estimate of $\mathcal R_n^{\mathcal F}$ w.h.p., giving
\begin{equation*}
    \mathcal R_T^b\le \Otil\left (s\sqrt d\sqrt{\sum_{t=1}^T \sigma_t^2}\log \frac 1\delta+s\log \frac 1\delta\right ).
\end{equation*}

So we only need to bound the regret for the ``commit'' phase, namely $\mathcal R_T^a$. As mentioned in the main text, we will consider the regret contributed from inside and outside $S$ seprately. Formally, we will write $\mathcal R_T^a$ as
\begin{equation*}
    \mathcal R_T^a=\sum_{\Delta=2^{-2},\ldots,\Delta_f}\sum_{t\in \mathcal T_\Delta^a} \left (\sum_{i\in S}\theta_i^\ast(\theta_i^\ast-x_{t,i})+\sum_{i\notin S}(\theta_i^\ast)^2\right ),
\end{equation*}

where the equality is because we will not put any mass on those $ i\notin S$ during the ``commit'' phase. We still assume the `good events' $\mathcal G_\Delta,\mathcal H_\Delta$ hold for all $\Delta$ (defined in \Cref{eq:good events}). By $\mathcal H_{2\Delta}$, for a given $\Delta$, $\sum_{i\not \in S}(\theta_i^\ast)^2\le 36s\Delta^2$. Moreover, by \Cref{lem:Zhang unknown is overestimation}, we will have
\begin{equation*}
    \mathcal R_T^a\le \sum_{\Delta=2^{-2},\ldots,\Delta_f}\bar{\mathcal R_{n_\Delta^a}^{\mathcal F}}+36s\sum_{\Delta=2^{-2},\ldots,\Delta_f} n_\Delta^a\Delta^2.
\end{equation*}

By the terminating criterion $\frac 1n \bar{\mathcal R_n^{\mathcal F}}<\Delta^2$ (\Cref{line:phase a terminiate} of \Cref{alg:framework}), for all $\Delta$, we will have $n_\Delta^a-1\le \frac{C}{\Delta^2}\bar{\mathcal R_{n_\Delta^a}^{\mathcal F}}$, which means
\begin{equation*}
    n_\Delta^a\le \frac{C}{\Delta^2}\left (s^{1.5}\sqrt{\sum_{t\in \tilde{\mathcal T_\Delta^a}} (r_t-\langle  x_t, \hat \beta\rangle)^2\ln \frac 1\delta}+s^2\sqrt{2\ln \frac{n_\Delta^a}{s\delta^2}\ln \frac 1\delta}+s^{1.5}\sqrt{2\ln \frac 1\delta}+s^2\ln \frac 1\delta\right )+1.
\end{equation*}

Therefore, plugging back into the expression of $\mathcal R_T^a$ gives (the term related with $\bar{\mathcal R_{n_\Delta^a}^{\mathcal F}}$ is dominated by the second term, as intuitively explained in the main text):
\begin{equation*}
    \mathcal R_T^a\le \sum_{\Delta=2^{-2},\ldots, \Delta_f}\Otil\left (s^{2.5}\sqrt{\sum_{t\in \tilde{\mathcal T_\Delta^a}} (r_t-\langle  x_t, \hat \beta\rangle)^2\ln \frac 1\delta}+s^3\left (\sqrt{\ln \frac 1\delta}+\ln \frac 1\delta\right )+s\Delta^2\right ).
\end{equation*}

The last term is simply bounded by $\O(s)$ after summing up over $\Delta$. Let us focus on the first term, where we need to upper bound the magnitude of $\bar{\mathcal R_n^{\mathcal F}}$. Applying \Cref{lem:ridge error bound} shows that the following with probability $1-\delta$:
\begin{equation*}
    \sum_{k=1}^n \left (r_k-\langle  x_k,\hat \beta\rangle\right )^2\le \sum_{k=1}^n \eta_k^2 + 2s \ln \frac{n}{s\delta^2}+2.
\end{equation*}

Therefore, we can write the regret (conditioning on this good event holds for all $\Delta$, which is with probability at least $1-dT\delta$ according to \Cref{thm:appendix total regret of part b}) as
\begin{equation}\label{eq:phase A VOFUL2 unknown}
    \mathcal R_T^a\le \sum_{\Delta=2^{-2},\ldots,\Delta_f}\Otil\left (s^{2.5}\sqrt{\sum_{t\in \mathcal T_\Delta^a}\eta_t^2\log \frac 1\delta}+s^3\left (\sqrt{\log \frac 1\delta}+\log \frac 1\delta\right )+s\Delta^2\right ).
\end{equation}

Again by the technique we used for \Cref{sec:appendix analysis}, we will need the following lower bound for all  $n_\Delta^a$ except for the last one, which follows from the fact that $\bar{\mathcal R_n^{\mathcal F}}\ge \mathcal R_n^{\mathcal F}$ (\Cref{lem:Zhang unknown is overestimation}):
\begin{equation*}
    n_\Delta^a\ge \frac{Cs^{1.5}}{\Delta^2}\sqrt{\sum_{t\in \tilde{\mathcal T_\Delta^a}} \eta_t^2\ln \frac 1\delta}.
\end{equation*}

By the summation technique that we used in \Cref{sec:appendix analysis}, more preciously, by the derivation of \Cref{eq:summation technique}, we will have the following derivation:
\begin{equation*}
    \sum_{\Delta\ne \Delta_f}Cs^{1.5}\sqrt{\sum_{t\in \mathcal T_\Delta^a}\eta_t^2 \ln \frac 1\delta}\le \sqrt{\log_4 X}\sqrt{\sum_{\frac 14\ge \Delta\ge \Delta_X}\sum_{t\in \mathcal T_\Delta^a}C^2s^3\eta_t^2 \ln \frac 1\delta}+\frac 1X \sum_{\Delta_X>\Delta>\Delta_f} \frac{Cs^{1.5}}{\Delta^2}\sqrt{\sum_{t\in \mathcal T_\Delta^a}\eta_t^2 \ln \frac 1\delta}
\end{equation*}

where $X$ is defined as
\begin{equation*}
    X=\left . T\middle /\sqrt{\sum_{\Delta\ne \Delta_f}\sum_{t\in \mathcal T_\Delta^a}C^2s^3\eta_t^2 \ln \frac 1\delta}\right .
\end{equation*}

and $\Delta_X=2^{-\lceil \log_4X\rceil}$, which means $\Delta_X^2\le \frac 1X$. Hence, we will have (the second summation will be bounded by $\frac TX$ as $\sum_\Delta n_\Delta^a\le T$)
\begin{equation*}
    \sum_{\Delta\ne \Delta_f}Cs^{1.5}\sqrt{\sum_{t\in \mathcal T_\Delta^a}\sigma_t^2 \ln \frac 1\delta}\le \Otil\left (\sqrt{\sum_{\Delta\ne \Delta_f}\sum_{t\in \mathcal T_\Delta^a}C^2s^3\eta_t^2 \log \frac 1\delta}\right )=\Otil\left (s^{1.5}\sqrt{\sum_{t=1}^T \eta_t^2\log \frac 1\delta}\right ).
\end{equation*}

In other words, the first term of \Cref{eq:phase A VOFUL2 unknown} will be bounded by
\begin{equation*}
    s^{2.5}\sum_{\Delta=2^{-2},\ldots,\Delta_f}\sqrt{\sum_{t\in \mathcal T_\Delta^a}\eta_t^2\ln \frac 1\delta}\le \Otil\left (s^{2.5}\sqrt{\sum_{t=1}^T \eta_t^2\log \frac 1\delta}\right ).
\end{equation*}

So we are done with the first and the last term of \Cref{eq:phase A VOFUL2 unknown}. Now consider the second term, which is equivalent to bounding $\sum_\Delta 1$. Use the following property guaranteed by (again) \Cref{line:phase a terminiate} of \Cref{alg:framework}:
\begin{equation*}
    \sum_{\Delta\ne \Delta_f}\Delta^{-2} C s^2 \ln \frac 4\delta\le \sum_{\Delta\ne \Delta_f}n_\Delta^a\le T,
\end{equation*}

which means $\sum_\Delta 1\le \log_4(T/s^2)=\Otil(1)$. Combining them together gives
\begin{equation*}
    \mathcal R_T^a\le \Otil\left (s^{2.5}\sqrt{\sum_{t=1}^T \eta_t^2\log \frac 1\delta}+s^3\left (\sqrt{\log \frac 1\delta}+\log \frac 1\delta\right )\right ),
\end{equation*}

At last, due to \Cref{thm:variance concentration raw}, with probability $1-\delta$,
\begin{equation*}
    \sum_{i=1}^n \left (r_i-\langle  x_i,\beta^\ast\rangle\right )^2=\sum_{k=1}^n \eta_k^2\overset{(a)}{\le} \sum_{k=1}^n \E[\eta_k^2\mid \mathcal F_{k-1}]+4\sqrt 2\sum_{k=1}^n \sigma_k^2\ln \frac 2\delta\le 8\sum_{k=1}^n \sigma_k^2 \ln \frac 2\delta,
\end{equation*}

where (a) is due to \Cref{thm:variance concentration raw} and (b) is due to $\E[\eta_k^2\mid \mathcal F_{k-1}]\le \sigma_k^2$. Hence,
\begin{equation*}
    \mathcal R_T^a\le \Otil\left (s^{2.5}\sqrt{\sum_{t=1}^T \sigma_t^2}\log \frac 1\delta+s^3\left (\sqrt{\log \frac 1\delta}+\log \frac 1\delta\right )\right ).
\end{equation*}

Combining this with the regret for $\mathcal R_T^b$ (\Cref{thm:appendix total regret of part b}) gives
\begin{align*}
    \mathcal R_T&\le \O\left ((s^{2.5}+s\sqrt d)\sqrt{\sum_{t=1}^T \sigma_t^t}\log \frac 1\delta+s^3\left (\sqrt{\log \frac 1\delta}+\log \frac 1\delta\right )+s\log \frac 1\delta\right )\\
    &=\O\left ((s^{2.5}+s\sqrt d)\sqrt{\sum_{t=1}^T \sigma_t^t}\log \frac 1\delta+s^3\log \frac 1\delta\right ),
\end{align*}

as claimed.
\end{proof}

\subsection{Regret Over-Estimation}\label{sec:VOFUL2 regret upper bound}
In this section, we first discuss why the strengthened \Cref{prop:VOFUL2 regret} holds. After that, we argue that our pessimistic estimation $\bar{\mathcal R_n^{\mathcal F}}$ is indeed an over-estimation of $\mathcal R_n^{\mathcal F}$. We state the \texttt{VOFUL2} algorithm in \Cref{alg:VOFUL2} and also restate \Cref{prop:VOFUL2 regret} as \Cref{prop:VOFUL2 regret appendix} for the ease of presentation.

\begin{algorithm}[!t]
\caption{\texttt{VOFUL2} Algorithm \citep{kim2021improved}}\label{alg:VOFUL2}
\begin{algorithmic}[1]
\For{$t=1,2,\ldots,T$}
\State Compute the action for the $t$-th round as
\begin{equation}\label{eq:VOFUL2 update}
    x_t=\operatornamewithlimits{argmax}_{x\in \mathcal X}\max_{\theta\in \bigcap_{s=1}^{t-1}\Theta_s}\langle x,\theta\rangle,
\end{equation}

where $\Theta_s$ is defined in \Cref{eq:VOFUL2 confidence set}.
\State Observe reward $r_t=\langle x_t,\theta^\ast\rangle+\eta_t$.
\EndFor
\end{algorithmic}
\end{algorithm}

\begin{proposition}\label{prop:VOFUL2 regret appendix}
\texttt{VOFUL2} on $d$ dimensions guarantees, with probability at least $1-\delta$,
\begin{equation}
    \mathcal R_T^{\mathcal F}= \tilde{\mathcal O}\left (d^{1.5}\sqrt{\sum_{t=1}^T\eta_t^2\log \frac 1\delta}+d^2\log \frac 1\delta\right )=\tilde{\mathcal O}\left (d^{1.5}\sqrt{\sum_{t=1}^T\sigma_t^2}\log \frac 1\delta+d^2\log \frac 1\delta\right ),
\end{equation}

where $T$ is a stopping time finite a.s. and $\sigma_1^2,\sigma_2^2,\ldots,\sigma_T^2$ are the variances of $\eta_1,\eta_2,\ldots,\eta_T$.
\end{proposition}
\begin{proof}
We will follow the proof of \citet{kim2021improved} and highlight the different steps. We first argue that an analog to their Empirical Bernstein Inequality \citep[Theorem 4]{zhang2021improved} still holds.
\begin{lemma}[{Analog of Theorem 4 from \citet{zhang2021improved}}]\label{lem:thm 4 from Zhang}
Let $\{X_i\}_{i=1}^n$ be a sequence of zero-mean Gaussian random variables such that $n$ is a stopping time finite a.s. Then for all $n\ge 8$ and any $\delta\in (0,1)$,
\begin{equation*}
    \Pr\left \{\left \lvert \sum_{i=1}^n X_i \right \rvert\le 8\sqrt{\sum_{i=1}^n X_i^2\ln \frac 4\delta}\right \}\ge 1-\delta.
\end{equation*}
\end{lemma}
\begin{proof}
    This is a direct corollary of \Cref{thm:Bernstein with common mean}.
\end{proof}

Thanks to \Cref{thm:variance concentration final}, the following analog of Lemma 17 from \citet{zhang2021improved} also holds:
\begin{lemma}[{Analog of Lemma 17 from \citet{zhang2021improved}}]\label{lem:lemma 17 from Zhang}
Let $\{X_i\}_{i=1}^n$ be a sequence of zero-mean Gaussian random variables such that $n$ is a stopping time finite a.s. Then for all $\delta\in (0,1)$,
\begin{equation*}
\Pr\left \{\sum_{i=1}^n X_i^2\ge 8\sum_{i=1}^n \sigma_i^2\ln \frac 2\delta\right \}\le \delta,
\end{equation*}

where $\sigma_i^2$ is the variance of $X_i$.
\end{lemma}

Then consider their Elliptical Potential Counting Lemma \citep[Lemma 5]{kim2021improved}, which holds as long as $\lVert x_t\rVert_2\le 1$. In our setting, we indeed have this property as $x_t\in \mathbb B^d$ (only the noises can be unbounded, instead of actions). Hence, this lemma still holds.
\begin{proposition}[{\citet[Lemma 5]{kim2021improved}}]
Let $x_1,x_2,\ldots,x_k\in \mathbb R^d$ be such that $\lVert x_i\rVert_2\le 1$ for all $s\in [k]$. Let $V_k=\lambda I+\sum_{i=1}^k x_ix_i^\trans$. Let $J=\{i\in [k]\mid \lVert x_i\rVert_{V_{i-1}^{-1}}^2\ge q\}$, then
\begin{equation*}
\lvert J\rvert\le \frac{2d}{\ln (1+q)} \ln \left (1+\frac{2/e}{\ln (1+q)}\frac 1\lambda\right ).
\end{equation*}
\end{proposition}

Then consider the confidence set construction:
\begin{equation}\label{eq:VOFUL2 confidence set}
\Theta_t=\bigcap_{\ell=1}^L\left \{\theta\in \mathbb B^d\middle \vert \left \lvert \sum_{s=1}^t \bar{(x_s^\trans \mu)}_\ell\epsilon_s(\theta) \right \rvert\le \sqrt{\sum_{s=1}^t\bar{(x_s^\trans \mu)}_\ell^2\epsilon_s^2(\theta)\iota},\forall \mu\in \mathbb B^d\right \},
\end{equation}

where $\epsilon_s(\theta)=r_s-x_s^\trans\theta$, $\epsilon_s^2(\theta)=(\epsilon_s(\theta))^2$, $\iota=128\ln ((12K2^L)^{d+2}/\delta)=\O(\sqrt d)$, $L=\max\{1,\lfloor \log_2(1+\frac Td)\rfloor\}$ and $\bar{(x)}_\ell=\min\{\lvert x\rvert,2^{-\ell}\}\frac{x}{\lvert x\rvert}$. By using \Cref{lem:thm 4 from Zhang} together with their original $\epsilon$-net coverage argument, we have $\theta^\ast\in \Theta_t$ for all $t\in [T]$ with high probability:
\begin{lemma}[Analog of Lemma 1 from {\citet{kim2021improved}}]\label{lem:Lemma 1 from Kim}
The good event $\mathcal E_1=\{\forall t\in [T],\theta^\ast\in \Theta_t\}$ happens with probability $1-\delta$.
\end{lemma}

Similar to \citet{kim2021improved}, we define $\theta_t$ be the maximizer of \Cref{eq:VOFUL2 update} in the $t$-th round and define $\mu_t=\theta_t-\theta^\ast$. We also consider the following good event
\begin{equation*}
    \mathcal E_2:\forall t\in [T],\sum_{s=1}^t \epsilon_s^2(\theta^\ast)\le 8\sum_{s=1}^t \sigma_s^2\ln \frac{8T}{\delta},
\end{equation*}

which happens with probability $1-\delta$ due to \Cref{lem:lemma 17 from Zhang}. Define
\begin{equation*}
W_{\ell,t-1}(\mu)=2^{-\ell}\lambda I+\sum_{s=1}^{t-1}\left (1\wedge \frac{2^{-\ell}}{\lvert x_s^\trans \mu\rvert}\right )x_sx_s^\trans.
\end{equation*}

Abbreviate $W_{\ell,t-1}(\mu_t)$ as $W_{\ell,t-1}$, then we have the following lemma, which is slightly different from the original one, whose proof will be presented later.
\begin{lemma}[{Analog of Lemma 4 from \citet{kim2021improved}}]\label{lem:Lemma 4 from Kim}
Conditioning on $\mathcal E_1$ and setting $\lambda=1$, we have
\begin{enumerate}[leftmargin=*]
    \item $\lVert \mu_t\rVert_{W_{\ell,t-1}}^2\le C_12^{-\ell}(\sqrt{A_{t-1}\iota}+\iota)$ for some absolute constant $C_1$, where $A_t\triangleq \sum_{s=1}^t \eta_s^2$.
    \item For all $s\le t$, we have $\lVert \mu_t\rVert_{W_{\ell,s-1}}^2\le C_12^{-\ell}(\sqrt{A_{s-1}\iota}+\iota)$.
    \item There exists absolute constant $C_2$ such that $x_t\mu_t\le C_2\lVert x_t^\trans\rVert_{W_{\ell,t-1}^{-1}}^2 (\sqrt{A_{t-1}\iota}+\iota)$.
\end{enumerate}
\end{lemma}

Therefore, as we have the same $\mathcal E_1$, Lemma 5 and a similar Lemma 4 (which uses $\eta_s$ instead of $\sigma_s$), we can conclude that
\begin{align*}
\mathcal R_T^{\mathcal F}&\le C\left (\sqrt{\sum_{s=1}^{T-1}\eta_s^2\iota}+\iota\right )d\ln^2 \left (1+C\left (\sqrt{\sum_{s=1}^{T-1}\eta_s^2\iota}+\iota\right )\left (1+\frac Td\right )^2\right )\\
&=\Otil\left (d^{1.5}\sqrt{\sum_{t=1}^T \eta_t^2\log \frac 1\delta}+d^2\log \frac 1\delta\right ),
\end{align*}

as in \citet[Theorem 2]{kim2021improved} (recall that $\iota=\O(\sqrt d)$). Conditioning on $\mathcal E_2$, we then further have
\begin{equation*}
    \mathcal R_T^{\mathcal F}\le \Otil\left (d^{1.5}\sqrt{\sum_{t=1}^T \sigma_t^2}\log \frac 1\delta+d^2\log \frac 1\delta\right ),
\end{equation*}

as claimed.
\end{proof}

\begin{proof}[Proof of \Cref{lem:Lemma 4 from Kim}]
By definition and abbreviating $\bar{(\cdot)}_\ell$ as $\bar{(\cdot)}$, we have
\begin{align*}
\lVert \mu_t\rVert_{W_{\ell,t-1}-2^{-\ell}\lambda I}^2
&=\sum_{s=1}^{t-1}\bar{(x_s^\trans \mu_t)} (x_s^\trans \mu_t)=\sum_{s=1}^{t-1}\bar{(x_s^\trans \mu_s)}(x_s\theta_t-r_t+r_t-x_s\theta^\ast)\\
&=\sum_{s=1}^{t-1}\bar{(x_s^\trans \mu_t)}(-\epsilon_s(\theta_t)+\epsilon_s(\theta^\ast))\\
&\overset{(\mathcal E_1)}{\le} \sqrt{\sum_{s=1}^{t-1}(x_s^\trans \mu_t)^2\epsilon_s^2 (\theta_t)\iota}+\sqrt{\sum_{s=1}^{t-1}\bar{(x_s^\trans \mu_t)}^2\epsilon_s^2 (\theta^\ast)\iota}\\
&\overset{(a)}{\le} \sqrt{\sum_{s=1}^{t-1}\bar{(x_s^\trans \mu_t)}^22(x_s^\trans \mu_t)^2\iota}+2\sqrt{\sum_{s=1}^{t-1}\bar{(x_s^\trans \mu_t)}^22\epsilon_s^2 (\theta^\ast)\iota}\\
&\overset{(b)}{\le} \sqrt{2^{-\ell}\sum_{s=1}^{t-1}\bar{(x_s^\trans \mu_t)}^22(x_s^\trans \mu_t)^2\iota}+2^{-\ell}16\sqrt{\sum_{s=1}^{t-1}\eta_s^2\iota}\\
&\le 2\sqrt{2\sum_{s=1}^{t-1}\bar{(x_s^\trans \mu_t)}(x_s^\trans \mu_t)\iota}+2^{-\ell}16\sqrt{\sum_{s=1}^{t-1}\eta_s^2\iota}\\
&=\sqrt{2^{-\ell}8\lVert \mu_t\rVert_{V_{t-1}-\lambda I}^2\iota}+2^{-\ell}16\sqrt{\sum_{s=1}^{t-1}\eta_s^2\iota},
\end{align*}

where (a) used $\epsilon_s^2(\theta_t)=(r_s-x_s\theta_t)^2=(x_s^\trans (\theta^\ast-\theta_t)+\epsilon_s^2(\theta^\ast))\le 2(x_s^\trans \mu_t)^2+2\epsilon_s^2$ and (b) used $\epsilon_s^2(\theta^\ast)=\eta_s^2$. By the self-bounding property \Cref{lem:self-bounding property from x<a+bsqrt(x)}, we have
\begin{equation*}
\lVert \mu_t\rVert_{V_{t-1}-\lambda I}^2\le 16\sqrt{\sum_{s=1}^{t-1}\sigma_s^2}+8\iota,
\end{equation*}

which means $\lVert \mu_t\rVert_{V_{t-1}}^2\le 4\lambda+16\sqrt{\sum_{s=1}^{t-1}\sigma_s^2\iota}+8\iota$. Setting $\lambda=1$ gives the first conclusion. Based on this, the second and third conclusion directly follow according to \citet{kim2021improved}.
\end{proof}

\subsection{Extension to Unknown Variance Cases}
Based on \Cref{prop:VOFUL2 regret}, we then show that, our regret estimation $\bar{\mathcal R_n^{\mathcal F}}$ \Cref{eq:regret estimation with unknown variance} is indeed pessimistic (i.e., \Cref{lem:Zhang unknown is overestimation}).
\begin{proof}[Proof of \Cref{lem:Zhang unknown is overestimation}]
From \Cref{lem:ridge error bound} with $\lambda=1$, with probability $1-\delta$, we will have (recall the assumption that $\sigma_t^2\le 1$ for all $t\in [T]$)
\begin{equation*}
    \sum_{k=1}^n (r_k- \langle  x_k, \beta^\ast\rangle)^2 =\sum_{k=1}^n\eta_k^2\le \sum_{k=1}^n (r_k-\langle  x_k, \hat \beta\rangle)^2+2s^2\ln \frac{n}{s \delta^2}+2.
\end{equation*}

Therefore we have
\begin{align*}
    \mathcal R_n^{\mathcal F}&\le C\left (s^{1.5}\sqrt{\sum_{k=1}^n \eta_k^2\ln \frac 1\delta}+s^2\ln \frac 1\delta\right )\\
    &\le C\left (s^{1.5}\sqrt{\left (\sum_{k=1}^n (r_k-\langle  x_k, \hat \beta\rangle)^2+2s\ln\frac{n}{s \delta^2}+2\right )\ln \frac 1\delta}+s^2\ln \frac 1\delta\right )\\
    &\le C\left (s^{1.5}\sqrt{\sum_{k=1}^n (r_k-\langle  x_k, \hat \beta\rangle)^2\ln \frac 1\delta}+s^2\sqrt{2\ln \frac{n}{s\delta^2}\ln \frac 1\delta}+s^{1.5}\sqrt{2\ln \frac 1\delta}+s^2\ln \frac 1\delta\right )=\bar{\mathcal R_n^{\mathcal F}}.
\end{align*}

In other words, our $\bar{\mathcal R_n^{\mathcal F}}$ is an over-estimation of $\mathcal R_n^{\mathcal F}$ with probability $1-\delta$.
\end{proof}

\section{Omitted Proof in \Cref{sec:zhou et al} (Analysis of \texttt{Weighted OFUL})}\label{sec:appendix zhou et al}

\subsection{Proof of Main Theorem}
Similar to the \texttt{VOFUL2} algorithm, we still assume \Cref{prop:Weighted OFUL regret} indeed holds and defer the discussions to the next section. We restate the regret guarantee of \texttt{Weighted OFUL} \citep{zhou2021nearly}, namely \Cref{thm:regret with Weighted OUFL}, for the ease of reading, as follows:

\begin{theorem}[Regret of \Cref{alg:framework} with \texttt{Weighted OFUL} in Known-Variance Case]\label{thm:appendix weighted OFUL known}
Consider \Cref{alg:framework} with $\mathcal F$ as \texttt{Weighted OFUL} \citep{zhou2021nearly} and $\bar{\mathcal R_n^{\mathcal F}}$ as
\begin{equation*}
    \bar{\mathcal R_n^{\mathcal F}}\triangleq C\left (\sqrt{sn\ln \frac 1\delta}+s\sqrt{\sum_{k=1}^n \sigma_k^2 \ln \frac 1\delta}\right ).
\end{equation*}

The algorithm ensures the following regret bound with probability $1-\delta$:
\begin{equation*}
    \mathcal R_T= \Otil\left ((s^2+s\sqrt d)\sqrt{\sum_{t=1}^T \sigma_t^2}\log \frac 1\delta+s^{1.5}\sqrt{T}\log \frac 1\delta\right ).
\end{equation*}
\end{theorem}
\begin{proof}
Firstly, from \Cref{prop:Weighted OFUL regret}, the condition of applying \Cref{thm:appendix total regret of part b} holds. Therefore, we have
\begin{equation*}
    \mathcal R_T^b\le \Otil\left (s\sqrt d\sqrt{\sum_{t=1}^T \sigma_t^2}\log \frac 1\delta+s\log \frac 1\delta\right ).
\end{equation*}

For $\mathcal R_T^a$, similar to \Cref{sec:appendix zhou et al}, we decompose it into two parts: those from $S$ and from outside of $S$. The former case is bounded by \Cref{prop:Weighted OFUL regret}, as
\begin{equation*}
    \text{Regret from }S\text{ with gap threshold }\Delta\le C\left (\sqrt{sn_\Delta^a\ln \frac 1\delta}+s\sqrt{\sum_{t\in \mathcal T_\Delta^a} \sigma_t^2 \ln \frac 1\delta}\right ).
\end{equation*}

For those outside $S$, we will bound it as $\O(sn_\Delta^a\Delta^2)$, where we only need to bound $n_\Delta^a$. From \Cref{line:phase a terminiate} of \Cref{alg:framework}, we have
\begin{align*}
    n_\Delta^a-1&\le \frac{C}{\Delta^2}\left (\sqrt{s(n_\Delta^a-1)\ln \frac 1\delta}+s\sqrt{\sum_{t\in \tilde{\mathcal T_\Delta^a}} \sigma_t^2 \ln \frac 1\delta}\right ).
\end{align*}

By the ``self-bounding'' property that $x\le a+b\sqrt x$ implies $x\le \O(a+b^2)$ (\Cref{lem:self-bounding property from x<a+bsqrt(x)}), we have
\begin{equation*}
    n_\Delta^a-1\le \frac{1}{\Delta^4} s\ln \frac 1\delta+\frac{s}{\Delta^2} \sqrt{\sum_{t\in \tilde{\mathcal T_\Delta^a}}\sigma_t^2\ln \frac 1\delta}.
\end{equation*}

Therefore, we can conclude that (the regret from $S$ is dominated)
\begin{equation*}
    \mathcal R_T^a\le \sum_{\Delta=2^{-2},\ldots} \O \left (\frac{s^2}{\Delta^2}\log \frac 1\delta+s^2\sqrt{\sum_{t\in \tilde{\mathcal T_\Delta^a}}\sigma_t^2\log \frac 1\delta}+s\Delta^2\right ),
\end{equation*}

where the last term is simply bounded by $\O(s)$. Again from \Cref{line:phase a terminiate} of \Cref{alg:framework}, we will have the following property for all $\Delta\ne \Delta_f$:
\begin{equation}\label{eq:lower bound Zhou}
    n_\Delta^a> \frac{C}{\Delta^2}\left (\sqrt{sn_\Delta^a\ln \frac 1\delta}+s\sqrt{\sum_{t\in {\mathcal T_\Delta^a}} \sigma_t^2 \ln \frac 1\delta}\right ).
\end{equation}

We first bound the second term, which basically follow the summation technique (\Cref{eq:summation technique}) that we used in \Cref{sec:appendix analysis,sec:appendix zhang et al unknown}:
\begin{equation*}
    \sum_{\Delta\ne \Delta_f}Cs\sqrt{\sum_{t\in \mathcal T_\Delta^a}\sigma_t^2 \ln \frac 1\delta}\le \sqrt{\log_4 X}\sqrt{\sum_{\frac 14\ge \Delta\ge \Delta_X}\sum_{t\in \mathcal T_\Delta^a}C^2s^2\sigma_t^2 \ln \frac 1\delta}+\frac 1X \sum_{\Delta_X>\Delta>\Delta_f} \frac{Cs}{\Delta^2}\sqrt{\sum_{t\in \mathcal T_\Delta^a}\sigma_t^2 \ln \frac 1\delta}
\end{equation*}

where $X$ is defined as
\begin{equation*}
    X=\left . T\middle /\sqrt{\sum_{\Delta\ne \Delta_f}\sum_{t\in \mathcal T_\Delta^a}C^2s^2\sigma_t^2 \ln \frac 1\delta}\right .
\end{equation*}

and $\Delta_X=2^{-\lceil \log_4X\rceil}$, which means $\Delta_X^2\le \frac 1X$. Hence, we will have (the second summation will be bounded by $\frac TX$ as $\sum_\Delta n_\Delta^a\le T$)
\begin{equation*}
    \sum_{\Delta\ne \Delta_f}Cs\sqrt{\sum_{t\in \mathcal T_\Delta^a}\sigma_t^2 \ln \frac 1\delta}\le \Otil\left (\sqrt{\sum_{\Delta\ne \Delta_f}\sum_{t\in \mathcal T_\Delta^a}C^2s^2\sigma_t^2 \log \frac 1\delta}\right )=\Otil\left (s\sqrt{\sum_{t=1}^T \sigma_t^2\log \frac 1\delta}\right ).
\end{equation*}

Hence, for the second term, we have
\begin{equation*}
    \O\left (s^2\sqrt{\sum_{t=1}^T \sigma_t^2 \log \frac 1\delta\log T}\right )=\Otil\left (s^2\sqrt{\sum_{t=1}^T \sigma_t^2 \log \frac 1\delta}\right ).
\end{equation*}

At last, we consider the first term. From the same lower bound of $n_\Delta^a$ (\Cref{eq:lower bound Zhou}), we will have
\begin{equation*}
    n_\Delta^a>\frac{C}{\Delta^2}\sqrt{sn_\Delta^a \ln \frac 1\delta}\Longrightarrow n_\Delta^a>\frac{C^2}{\Delta^4} s\ln \frac 1\delta.
\end{equation*}

By the fact that $\sum_{\Delta\ne \Delta_f}n_\Delta^a\le T$, we will have
\begin{equation*}
    T\ge C^2 s\ln \frac 1\delta \sum_{\Delta\ne \Delta_f} \Delta^{-4}=\O(1) C^2 s\ln \frac 1\delta (2\Delta_f)^{-4}.
\end{equation*}

Henceforth,
\begin{equation*}
    \O\left (s^2\log \frac 1\delta\cdot \sum_{\Delta=2^{-2},\ldots,\Delta_f}\Delta^{-2}\right )=\O\left (s^2\log \frac 1\delta\Delta_f^{-2}\right )\le \O\left (s^2\log \frac 1\delta\sqrt{\frac{T}{s\ln \frac 1\delta}}\right )=\O\left (s^{1.5}\sqrt{T\log \frac 1\delta}\right ).
\end{equation*}

Combining all above together gives
\begin{align*}
    \mathcal R_T=\mathcal R_T^a+\mathcal R_T^b&\le \O\left (s^{1.5}\sqrt{T\log \frac 1\delta}+s^2\sqrt{\sum_{t=1}^T \sigma_t^2\log \frac 1\delta}+s\right )+\Otil\left (s\sqrt d\sqrt{\sum_{t=1}^T \sigma_t^2}\log \frac 1\delta+s\log \frac 1\delta\right )\\
    &\le \Otil\left ((s^2+s\sqrt d)\sqrt{\sum_{t=1}^T \sigma_t^2}\log \frac 1\delta+s^{1.5}\sqrt{T}\log \frac 1\delta\right ),
\end{align*}

as claimed.
\end{proof}

\subsection{Regret Over-estimation}
We again briefly argue that \Cref{prop:Weighted OFUL regret} holds under our noise model. We present their algorithm in \Cref{alg:Weighted OFUL}.

\begin{algorithm}[!t]
\caption{\texttt{Weighted OFUL} Algorithm \citep{zhou2021nearly}}\label{alg:Weighted OFUL}
\begin{algorithmic}[1]
\State Intialize $A_0\gets \lambda I$, $c_0\gets 0$, $\hat \theta_0\gets A_0^{-1}c_0$, $\hat \beta_0=0$ and $\Theta_0\gets \{\theta\mid \lVert \theta-\hat \theta_0\rVert_{A_0}\le \hat \beta_0+\sqrt \lambda B\}$.
\For{$t=1,2,\ldots,T$}
\State Compute the action for the $t$-th round as
\begin{equation}\label{eq:Weighted OFUL update}
    x_t=\operatornamewithlimits{argmax}_{x\in \mathcal X}\max_{\theta\in \Theta_{t-1}}\langle x,\theta\rangle.
\end{equation}

\State Observe reward $r_t=\langle x_t,\theta^\ast\rangle+\eta_t$ and variance information $\sigma_t^2$, set $\bar \sigma_t=\max\{1/\sqrt d,\sigma_t\}$, set confidence radius $\hat \beta_t$ as
\begin{equation}\label{eq:Weighted OFUL radius}
    \hat \beta_t=8\sqrt{d\ln \left (1+\frac{t}{d\lambda \bar \sigma_{\min,t}^2}\right )\ln \frac{4t^2}{\delta}},
\end{equation}

where $\bar \sigma_{\min,t}\triangleq \min_{s=1}^{t} \bar \sigma_s$.

\State Calculate $A_t\gets A_{t-1}+x_tx_t^\trans / \bar \sigma_t^2$, $c_t\gets c_{t-1}+r_tx_t/\bar \sigma_t^2$, $\hat \theta_t\gets A_t^{-1}c_t$ and $\Theta_t\gets \{\theta\mid \lVert \theta-\hat \theta_t\rVert_{A_t}\le \hat \beta_t+\sqrt \lambda B\}$.
\EndFor
\end{algorithmic}
\end{algorithm}

\begin{proposition}
With probability at least $1-\delta$, \texttt{Weighted OFUL} executed for $T$ steps on $d$ dimensions guarantees
\begin{equation*}
\mathcal R_T^{\mathcal F}\le C\left (\sqrt{dT\log \frac 1\delta}+d\sqrt{\sum_{t=1}^T \sigma_t^2\log \frac 1\delta}\right )
\end{equation*}

where $C=\Otil(1)$, $T$ is a stopping time finite a.s., and $\sigma_1^2,\sigma_2^2,\ldots,\sigma_T^2$ are the variances of $\eta_1,\eta_2,\ldots,\eta_T$.
\end{proposition}
\begin{proof}[Proof Sketch]
We mainly follow the original proof by \citet{zhou2021nearly} and highlight the differences. We first highlight their Bernstein Inequality for vector-valued martingales also holds under our assumptions, as:
\begin{lemma}[{Analog of Theorem 4.1 from \citet{zhou2021nearly}}]\label{lem:thm 4.1 from Zhou}
Let $\{x_i\}_{i=1}^n$ be sequence of $d$-dimensional random vectoes such that $\lVert x_t\rVert_2\le L$. Let $\{\eta_i\}_{i=1}^n$ be a sequence of independent, symmetric and $\{\sigma_i^2\}_{i=1}^n$-sub-Gaussian random variables.
Let $r_t=\langle \theta^\ast,x_t\rangle+\eta_t$ for all $t\in [n]$. Set $Z_t=\lambda I+\sum_{s=1}^t x_sx_s^\trans$, $b_t=\sum_{s=1}^t r_sx_s$ and $\theta_t=Z_t^{-1}b_t$. Let $n$ be a stopping time finite a.s. Then, $\forall \delta\in (0,1)$,
\begin{equation*}
\Pr\left \{\left \lVert \sum_{s=1}^t x_s\eta_s\right \rVert_{Z_t^{-1}}\le \beta_t,\lVert \theta_t-\theta^\ast\rVert_{Z_t}\le \beta_t+\sqrt \lambda \lVert \theta^\ast\rVert_2,\forall t\in [n]\right \}\ge 1-\delta,
\end{equation*}

where $\beta_t=8\sigma\sqrt{d\ln (1+\frac{tL^2}{d\lambda})\ln \frac{8t^2}{\delta}}$.
\end{lemma}

The proof, which will be presented later, mainly follows from the idea of their proof of Theorem 4.1, except that we are using \Cref{thm:generalized Freedman}. Check the proof below for more details about this.

With this theorem, we can consequently conclude that the confidence construction is indeed valid by applying \Cref{lem:thm 4.1 from Zhou} to the sequence $\{\eta_t/\bar \sigma_t\}_{t\in [T]}$, which gives
\begin{equation*}
\lVert \hat \theta_t-\theta^\ast\rVert_{A_t}\le \hat \beta_t+\sqrt \lambda \lVert \theta^\ast\rVert_2\le \hat \beta_t+\sqrt \lambda,\quad \forall t\in [T],\quad \text{with probability }1-\delta,
\end{equation*}

where $\hat \beta_t=8\sqrt{d\ln (1+\frac{t}{d\lambda \bar \sigma_{\min,t}^2})\ln \frac{4t^2}{\delta}}$, as defined in \Cref{eq:Weighted OFUL radius}. Therefore, we can conclude their (B.19) from exactly the same argument, namely
\begin{equation*}
\mathcal R_T^{\mathcal F}\le 2\sum_{t=1}^T \min\left \{1,\bar \sigma_t(\hat \beta_{t-1}+\sqrt \lambda)\lVert x_t/\bar \sigma_t\rVert_{A_{t-1}^{-1}}\right \}.
\end{equation*}

Similar to their proof, define $\mathcal I_1=\{t\in [T]\mid \lVert x_t/\bar \sigma_t\rVert_{A_{t-1}^{-1}}\ge 1\}$ and $\mathcal I_2=[T]\setminus \mathcal I_1$, then we have (where $\bar \sigma_{\min}$ is the abbreviation of $\bar \sigma_{\min,T}=\min_{s=1}^T \bar \sigma_s$)
\begin{equation*}
\lvert \mathcal I_1\rvert=\sum_{t\in \mathcal I_1}\min\{1,\lVert x_t/\bar \sigma_t\rVert_{A_{t-1}^{-1}}^2\}
\le \sum_{t=1}^T\min\{1,\lVert x_t/\bar \sigma_t\rVert_{A_{t-1}^{-1}}^2\}\le 2d \ln \left (1+\frac{T}{d\lambda \bar \sigma_{\min}^2}\right ),
\end{equation*}

where the last step uses \Cref{lem:lem 11 from abbasi} and the fact that $\lVert x_t/\bar \sigma_t\rVert_2\le \bar \sigma_{\min}^{-1}$. Therefore, we are having the same Eq. (B.21) as theirs, which gives
\begin{equation*}
\mathcal R_T^{\mathcal F}\le 2\sqrt{2d\ln \left (1+\frac{T}{d\lambda \bar \sigma_{\min}^2}\right )}\sqrt{\sum_{t=1}^T (\hat \beta_{t-1}+\sqrt \lambda)^2\bar \sigma_t^2}+4d\ln \left (1+\frac{T}{d\lambda \bar \sigma_{\min}^2}\right ).
\end{equation*}

By the choice of $\bar \sigma_t=\max\{1/\sqrt d,\sigma_t\}$ and $\lambda=1$, we have
\begin{equation*}
\ln \left (1+\frac{T}{d\lambda \bar \sigma_{\min}^2}\right )\le \ln \left (1+\frac{T}{d}\right )=\Otil(1),
\end{equation*}

and that
\begin{equation*}
\hat \beta_t+\sqrt \lambda=\O\left (\sqrt{d\log T\log \frac T\delta}+1\right )=\Otil\left (\sqrt{d\log \frac 1\delta}\right ),
\end{equation*}

which gives
\begin{align*}
\mathcal R_T^{\mathcal F}
&=\Otil\left (d\sqrt{\sum_{t=1}^T \bar \sigma_t^2\log \frac 1\delta}\right )=\Otil\left (d\sqrt{\sum_{t=1}^T \left (\frac 1d+\sigma_t^2\right )\log \frac 1\delta}\right )\\
&=\Otil\left (\sqrt{dT\log \frac 1\delta}+d\sqrt{\sum_{t=1}^T \sigma_t^2\log \frac 1\delta}\right ),
\end{align*}

as claimed, while the second step uses $\bar \sigma_t^2=\min\{\frac 1d,\sigma_t^2\}\le \frac 1d+\sigma_t^2$.
\end{proof}

\begin{proof}[Proof of \Cref{lem:thm 4.1 from Zhou}]
Their original proof mainly use the following two auxiliary results: The first one is the well-known Freedman inequality \citep{freedman1975tail}, which is originally for bounded martingale difference sequences, while the second one is Lemma 11 from \citet{abbasi2011improved}. For the former one, from its variant for sub-Gaussian random variables (\Cref{thm:generalized Freedman}), we have:
\begin{corollary}\label{eq:Freedman from Zhou}
Suppose that $\{\xi_i\}_{i=1}^n$ is a sequence of zero-mean random variables where $\xi_i\sim \text{subG}(\sigma_i^2)$ for some sequence $\{\sigma_i\}_{i=1}^n$. Let $n$ be a stopping time finite a.s. Then for all $x,v>0$ and $\lambda>0$,
\begin{equation*}
    \Pr\left \{\exists 1\le k\le n:\sum_{i=1}^k\xi_i\ge x\wedge \sum_{i=1}^k V_i\le v^2\right \}\le \exp\left (-\frac{x^2}{2v^2}\right ).
\end{equation*}

Moreover, for any $\delta\in (0,1)$, with probability $1-\delta$, we have
\begin{equation*}
\sum_{i=1}^n \xi_i\le 2\sqrt{\sum_{i=1}^n\sigma_i^2\ln \frac 2\delta}.
\end{equation*}
\end{corollary}
\begin{proof}
The first conclusion is done by applying \Cref{thm:generalized Freedman} optimally with $f(\lambda)=\frac 12 \lambda^2$, $V_i=\sigma_i^2$ and $\lambda=\frac{x}{v^2}$. The second conclusion is consequently proved by taking $v^2=\sum_{i=1}^n \sigma_i^2$ and $x=2v\sqrt{\ln \frac 2\delta}$.
\end{proof}

For their second auxiliary lemma \citep[Lemma 11]{abbasi2011improved}, one can see that the original lemma indeed holds for sub-Gaussian random variables. Therefore, we still have the following lemma:
\begin{proposition}[{\citet[Lemma 11]{abbasi2011improved}}]\label{lem:lem 11 from abbasi}
Let $\{x_t\}_{t=1}^T$ be a sequence in $\mathbb R^d$ and define $V_t=\lambda I+\sum_{s=1}^t x_sx_s^\trans$ for some $\lambda>0$. Then, if we have $\lVert x_t\rVert_2\le L$ for all $t\in [T]$, then
\begin{equation*}
\sum_{t=1}^T \min\left \{1,\lVert x_t\rVert_{V_{t-1}^{-1}}^2\right \}\le 2 \log \frac{\det(V_t)}{\det(\lambda I)}\le 2d\ln \frac{d\lambda+TL^2}{d\lambda}.
\end{equation*}
\end{proposition}

Recall the definition that $Z_t=\lambda I+\sum_{s=1}^t x_sx_s^\trans$. Further define $d_t=\sum_{s=1}^t x_i\eta_i$, $w_t=\lVert x_t\rVert_{Z_{t-1}^{-1}}$ and let $\mathcal E_t$ be the event that $\lVert d_s\rVert_{Z_{s-1}^{-1}}\le \beta_s$ for all $s\le t$. They proved the following lemma, which still applies to our case:
\begin{lemma}[{Analog of Lemma B.3 from \citet{zhou2021nearly}}]\label{lem:lem B.3 from Zhou}
With probability $1-\frac \delta 2$, with the definitions of $x_t$ and $\eta_t$ in \Cref{lem:thm 4.1 from Zhou}, the following inequality holds for all $t\ge 1$:
\begin{equation*}
    \sum_{s=1}^t \frac{2\eta_s x_s^\trans Z_{s-1}^{-1}d_{s-1}}{1+w_s^2}\mathbbm 1[\mathcal E_{s-1}]\le \frac 34\beta_t^2.
\end{equation*}
\end{lemma}
\begin{proof}
We only need to verify whether we can apply our Freedman's inequality to $\ell_s\triangleq \frac{2\eta_sx_s^\trans Z_{s-1}^{-1} d_{s-1}}{1+w_s^2}\mathbbm 1[\mathcal E_{s-1}]$. It is obvious that $\E[\ell_s\mid \mathcal F_{s-1}]=0$. Moreover, from the following inequality (which is their Eq. (B.3))
\begin{equation*}
\lvert \ell_s\rvert\le \frac{2\lVert x_s\rVert_{Z_{s-1}^{-1}}}{1+w_s^2}\lVert d_{s-1}\rVert_{Z_{s-1}^{-1}}\mathbbm 1[\mathcal E_{s-1}]\le \frac{2w_i}{1+w_i^2}\beta_{s-1}\le \min\{1,2w_i\}\beta_{i-1},
\end{equation*}

and the fact that $\eta_s\mid \mathcal F_{s-1}\sim \text{subG}(\sigma^2)$, we have $\ell_s\mid \mathcal F_{s-1}\sim \text{subG}((\sigma\beta_{s-1}\min\{1,2w_s\})^2)$. Denote the sub-Gaussian parameter as $\tilde \sigma_s$ for simplicity. We have
\begin{equation*}
\sum_{s=1}^t \tilde \sigma_s^2\le \sigma^2 \beta_t^2\sum_{s=1}^t (\min\{1,2w_s\})^2\le 4\sigma^2 \beta_t^2\sum_{s=1}^t \min\{1,w_s^2\}\le 8\sigma^2 \beta_t^2d\ln \left (1+\frac{tL^2}{d\lambda}\right ),
\end{equation*}

where the first inequality is due to the non-decreasing property of $\{\beta_s\}$ and the last one is due to \Cref{lem:lem 11 from abbasi}. Therefore, from our Freedman's inequality (\Cref{eq:Freedman from Zhou}), we can conclude that with probability $1-\delta/(4t^2)$,
\begin{align*}
\sum_{s=1}^t \ell_s&\le 2\sqrt{\sum_{s=1}^t8\sigma^2\beta_t^2d\ln \left (1+\frac{tL^2}{d\lambda}\right )\ln \frac{8t^2}{\delta}}\\
&\le \frac{\beta_t^2}{2}+16\sigma^2 d\ln \left (1+\frac{tL^2}{d\lambda}\right )\ln \frac{8t^2}{\delta}=\frac 34 \beta_t^2.
\end{align*}

Taking a union bound over $t$ and make use of the fact that $\sum_{t=1}^\infty t^{-2}<2$ completes the proof.
\end{proof}

We also have the following lemma:
\begin{lemma}[{Analog of Lemma B.4 from \citet{zhou2021nearly}}]
Under the same conditions as the previous lemma, with probability $1-\frac \delta 2$, we will have the following for all $\ge 1$ simultaneously:
\begin{equation*}
\sum_{s=1}^t \frac{\eta_s^2w_s^2}{1+w_s^2}\le \frac 14\beta_t^2.
\end{equation*}
\end{lemma}
\begin{proof}
We still need to apply our Freedman's inequality (\Cref{eq:Freedman from Zhou}) to
\begin{equation*}
\ell_s=\E\left [\frac{\eta_s^2w_s^2}{1+w_s^2}\middle \vert \mathcal F_{s-1}\right ]-\frac{\eta_s^2w_s^2}{1+w_s^2}.
\end{equation*}

As $\eta_s\mid \mathcal F_{s-1}\sim \text{subG}(\sigma^2)$, $\eta_s^2\mid \mathcal F_{s-1}$ is a sub-exponential random variable (\Cref{thm:sub-exponential}) such that
\begin{equation*}
\E\left [\exp(\lambda(\eta_s^2-\E[\eta_s^2]))\middle \vert \mathcal F_{s-1}\right ]\le \exp(16\lambda^2 \sigma^4),\quad \forall \lvert \lambda\rvert\le \frac{1}{4\sigma^2},
\end{equation*}

which consequently means
\begin{equation*}
\E\left [\exp(\lambda\ell_s)\middle \vert \mathcal F_{s-1}\right ]\le \exp\left (16\lambda^2\sigma^4(\min\{1,w_s^2\})^2\right ),\quad \forall \lvert \lambda\rvert\le \frac{1}{4\sigma^2}.
\end{equation*}

where the second step used the fact that $\frac{w_s^2}{1+w_s^2}\le \min\{1,w_s^2\}$ as we used in the proof of \Cref{lem:lem B.3 from Zhou}. Again by \Cref{lem:lem 11 from abbasi}, we can conclude that
\begin{equation*}
\sum_{s=1}^t \sigma^4(\min\{1,w_s^2\})^2\le 2\sigma^2d\ln \left (1+\frac{tL^2}{d\lambda}\right ).
\end{equation*}

Then, as we did in the proof of \Cref{thm:variance concentration raw}, we will apply \Cref{thm:generalized Freedman} to the martingale difference sequence $\{\ell_s\}_{s=1}^t$ with $V_s=\sigma^4(\min\{1,w_s^2\})^2$, $f(\lambda)=16\lambda^2$ for $\lambda < \frac{1}{4\sigma^2}$ and $f(\lambda)=\infty$ otherwise. Then for all $x,v>0$ and $\lambda\in (0,\frac{1}{\sigma^2})$, we have
\begin{equation*}
    \Pr\left \{\sum_{s=1}^t \ell_s>x\wedge \sqrt{\sum_{i=1}^n \sigma_i^4 (\min\{1,w_s^2\})^2}\le v\right \}\le \exp \left (-\lambda x+16\lambda^2 v^2\right ).
\end{equation*}

Picking $v^2=\sum_{s=1}^t \sigma_i^4 (\min\{1,w_s^2\})^2$ and $x=4\sqrt 2 v\sqrt{\ln \frac 2\delta}$ gives
\begin{align*}
    \Pr\left \{\sum_{s=1}^t \ell_s>4\sqrt{2\sum_{s=1}^t \sigma_i^4(\min\{1,w_s^2\})^2\ln \frac 2\delta}\right \}\le \exp \left (-\frac{x^2}{32v^2}\right )=\frac \delta 2,
\end{align*}

where $\lambda$ is set to $\frac{x}{32v^2}<\frac{1}{\sigma^2}$. Hence, with probability $1-\frac{\delta}{4t^2}$, we indeed have
\begin{equation*}
\sum_{s=1}^t \ell_s\le 4\sqrt{2\sum_{s=1}^t \sigma_i^4(\min\{1,w_s^2\})^2\ln \frac 2\delta}\le 8\sqrt{2\sigma^2d\ln \left (1+\frac{tL^2}{d\lambda}\right )\ln \frac{8t^2}{\delta}}.
\end{equation*}

Moreover, due to \Cref{lem:lem 11 from abbasi}, we have
\begin{equation*}
\sum_{s=1}^t \E\left [\frac{\eta_s^2w_s^2}{1+w_s^2}\middle \vert \mathcal F_{s-1}\right ]\le \sigma^2\sum_{s=1}^t \frac{w_s^2}{1+w_s}\le 2\sigma^2 d \ln \left (1+\frac{tL^2}{d\lambda}\right ),
\end{equation*}

which means, as $\frac{\eta_s^2w_s^2}{1+w_s^2}\le \ell_s+\E[\frac{\eta_s^2w_s^2}{1+w_s^2}]$,
\begin{equation*}
\sum_{s=1}^t \frac{\eta_s^2w_s^2}{1+w_s^2}\le 2\sigma^2 d \ln \left (1+\frac{tL^2}{d\lambda}\right )+8\sqrt{2\sigma^2d\ln \left (1+\frac{tL^2}{d\lambda}\right )\ln \frac{8t^2}{\delta}}\le \frac 14 \beta_t^2.
\end{equation*}

Taking a union bound over all $t$ and again making use of the fact that $\sum_{t=1}^\infty t^{-2}<2$ gives our conclusion.
\end{proof}

Therefore, as long as their Lemmas B.3 and B.4 still hold, we can conclude exactly the same conclusion from their derivation. One may refer to their proof for the details.
\end{proof}

\section{Auxilliary Lemmas}

\begin{lemma}[Self Bounding Inequality, {\citet[Lemma 38]{efroni2020exploration}}]\label{lem:self-bounding property from x<a+bsqrt(x)}
Let $0\le x\le a+b\sqrt x$ where $a,b,x\ge 0$, then we have
\begin{equation*}
    x\le 4a+2b^2.
\end{equation*}
\end{lemma}
\begin{proof}
As $x-b\sqrt x-a\le 0$, we have
\begin{equation*}
    \sqrt x\le \frac b2+\sqrt{\frac 14 b^2+4a}\le \frac b2+\sqrt{\frac{b^2}{4}}+\sqrt{4a}=b+2\sqrt a
\end{equation*}

from the fact that $\sqrt{a+b}\le \sqrt a+\sqrt b$. As $\sqrt x\ge 0$, we have
\begin{equation*}
    x\le (b+2\sqrt a)^2\le 2b^2+4a
\end{equation*}

due to the relation that $(a+b)^2\le 2a^2+2b^2$.
\end{proof}

\end{document}